\documentclass[11pt]{article}
\usepackage[T1]{fontenc}
\usepackage{lmodern}
\usepackage{inconsolata}
\AtBeginDocument{\begingroup\ttfamily\endgroup\DeclareFontShape{T1}{zi4}{m}{it}{<->ssub*zi4/m/n}{}}
\usepackage{fullpage}
\usepackage[authoryear,round]{natbib}
\usepackage{iftex}
\ifPDFTeX
  \usepackage[protrusion=true,expansion=true]{microtype}
\else
  \usepackage[protrusion=true,expansion=false]{microtype}
\fi
\usepackage{amsmath,amssymb,amsthm,mathtools,bm}
\usepackage{enumitem}
\usepackage{graphicx}
\usepackage{multirow}
\usepackage{booktabs}
\usepackage{caption}
\usepackage{float}
\usepackage[dvipsnames]{xcolor}
\usepackage[colorlinks,linkcolor=red,anchorcolor=blue,citecolor=blue]{hyperref}
\usepackage{xurl}
\renewcommand{\arraystretch}{1.15}
\numberwithin{equation}{section}

\newtheorem{theorem}{Theorem}[section]
\newtheorem{proposition}[theorem]{Proposition}
\newtheorem{lemma}[theorem]{Lemma}

\theoremstyle{definition}

\theoremstyle{remark}

\newcommand{\E}{\mathbb E}
\newcommand{\KL}{D_{\mathrm{KL}}}
\newcommand{\cA}{\mathcal A}
\newcommand{\cF}{\mathcal F}
\newcommand{\cS}{\mathcal S}
\newcommand{\cU}{\mathcal U}
\newcommand{\cV}{\mathcal V}
\newcommand{\ip}[2]{\left\langle #1,#2\right\rangle}
\newcommand{\norm}[1]{\left\lVert #1\right\rVert}
\newcommand{\R}{\mathbb R}

\newif\ifshowcomments
\showcommentsfalse

\newcommand{\QwenRLMath}{Qwen3-4B-RL-Math}
\newcommand{\QwenRLCode}{Qwen3-4B-RL-Code}
\newcommand{\SkyworkTeacher}{Skywork-7B}
\newcommand{\QwenOneSevenNT}{Qwen3-1.7B (Non-thinking)}
\newcommand{\QwenFourNT}{Qwen3-4B (Non-thinking)}
\newcommand{\QwenFourteenNT}{Qwen3-14B (Non-thinking)}

\title{Why On-Policy Distillation Sometimes Fails:\\Vanishing Learning Signals}

\author{%
  Lei Zhao\textsuperscript{1}\thanks{Contributed equally to this work. Email: \texttt{leizhao7@upenn.edu}.}%
  \quad Qichao Zhao\textsuperscript{2}\protect\footnotemark[1]%
  \quad Bowen Zuo\textsuperscript{3}%
  \quad Qishi Zhan\textsuperscript{4}\\[0.6em]
  {\normalsize\textsuperscript{1}\,University of Pennsylvania\quad
  \textsuperscript{2}\,Tsinghua University}\\
  {\normalsize\textsuperscript{3}\,University of California, Riverside\quad
  \textsuperscript{4}\,Marquette University}
}
\date{}

\begin{document}
\maketitle

\begin{abstract}
On-policy distillation (OPD) enables effective capability transfer between
language models, yet the mechanisms underlying its failures are not fully
understood. Across code generation and mathematical reasoning, OPD with
larger-scale teachers exhibits early loss plateaus, with an average final
loss reduction of 25.1\% after 200 updates, compared with 96.2\% for
self-RL teachers, obtained by further reinforcement learning (RL)
training of the initial student. To
understand this difference, we analyze OPD as an idealized continuous-time
dynamical system in the small-learning-rate limit. Our training-log diagnostics associate
these plateaus with an early decline in a gradient-based learning-signal
proxy while substantial loss remains; these measurements do not establish
why the underlying gradient weakens. We further prove a local
recovery guarantee for teachers sufficiently close to the initial student
in a shared parameterization under regularity conditions, offering a
conditional explanation for the success of self-RL teachers in our
experiments. Across runs with and without loss plateaus, we observe small
relative parameter changes (0.025--0.098\%) and high similarity between the
student's representations before and after OPD (linear CKA $>0.98$ across
layers). These observations suggest that limited representation adaptation
may contribute to learning-signal collapse, a hypothesis that remains to be
tested.
Code is available at
\url{https://github.com/leizhao7/opd-learning-signals}.
\end{abstract}

\section{Introduction}
\label{sec:intro}

On-policy distillation (OPD) has become an important component of
language-model post-training
\citep{yang2025qwen3,xiao2026mimo,lu2025onpolicy}. In OPD, the student
generates its own responses and is trained to match the teacher's next-token
distributions along these responses \citep{agarwal2024policy}.
This provides dense supervision on the prefixes the student actually visits,
reducing the train--inference distribution mismatch that motivates on-policy
imitation learning \citep{ross2011reduction,agarwal2024policy}. This approach has
shown promise for transferring capabilities between language models
\citep{agarwal2024policy,gu2024minillm,yang2025qwen3,xiao2026mimo}.
Despite these successes, recent studies have observed that
OPD does not consistently improve student performance, even when the teacher
substantially outperforms the student \citep{li2026rethinking}. These
findings suggest that teacher capability alone is insufficient to ensure
effective transfer, motivating a closer examination of when and why OPD
fails.

To investigate these failures, we examine OPD training dynamics across code
generation and mathematical reasoning (Section~\ref{sec:phenomenon}). We
observe that \emph{self-RL teachers}, obtained by further reinforcement learning (RL)
training of the initial student, support sustained loss reduction and substantial
improvements in student performance. By contrast, OPD with the larger-scale
teachers in our experiments exhibits early loss plateaus and more limited
performance gains (Figure~\ref{fig:teacher-trainability}). In these runs,
the distillation loss stops decreasing while substantial teacher--student
mismatch remains. Related observations have also been reported elsewhere:
\citet{li2026rethinking} observe weak gradients and persistent
token-probability discrepancies in a failing run, while
\citet{jiang2026cornerstones} report persistent
high-loss tokens after the aggregate OPD objective saturates. Together with
our results, these reports indicate that limited optimization progress
despite unresolved disagreement occurs across multiple OPD settings. This
raises a central question:
\begin{center}
\emph{Why does OPD stop reducing its loss before the student matches the teacher?}
\end{center}

Prior work has proposed several explanations for why OPD does not always
improve student performance
\citep{li2026rethinking,fu2026revisiting,wang2026demystifying}.
These include teacher--student incompatibility
\citep{li2026rethinking,wang2026not}, unreliable teacher guidance
on student-generated prefixes \citep{fu2026revisiting,zhu2026many},
limitations of sampled-token supervision
\citep{fu2026revisiting,jia2026asymmetric}, and imperfect alignment between
distillation updates and task success
\citep{wang2026demystifying,armandpour2026unmasking}. These studies show that a strong teacher does not always provide feedback
that helps the student learn. However, these explanations do not fully clarify why the OPD loss keeps
decreasing with some teachers but plateaus early with others.

We address this problem in the small-learning-rate limit
\citep{borkar2000ode}. In this setting, we model OPD as an idealized
continuous-time dynamical system and analyze its loss dynamics
(Section~\ref{sec:loss-dynamics}).
Our main contributions are as follows:
\begin{itemize}[leftmargin=*,itemsep=3pt,topsep=4pt,parsep=0pt]
  \item We identify a recurring failure mode of OPD across models and tasks:
  training stalls while much of the distillation loss remains. This pattern
  occurs with the larger-scale teachers in our experiments
  (Section~\ref{sec:phenomenon}).

  \item We diagnose an early decline in a training-log proxy for the learning
  signal while substantial loss remains, which we describe as
  \emph{premature learning-signal collapse}
  (Section~\ref{sec:learning-signal-collapse}).

  \item We prove a local recovery guarantee for OPD with teachers
  sufficiently close to the initial student in a shared parameterization
  (Section~\ref{sec:parameter-recovery}). This result offers a conditional
  theoretical explanation for the success of OPD in our experiments when
  self-RL teachers are used.

  \item Across runs with and without loss plateaus, we observe small
  relative parameter changes and high similarity between the student's
  representations before and after OPD
  (Section~\ref{sec:parameter-representation}). These observations motivate the hypothesis that limited adaptation may
  constrain teacher matching; they do not establish a causal explanation
  for the observed plateaus.
\end{itemize}

\section{Related Work}
\label{sec:related-work}

\paragraph{On-policy distillation and its failures.}
OPD trains students on their own rollouts, extending on-policy imitation
learning \citep{ross2011reduction} to language models
\citep{lin2020autoregressive,agarwal2024policy,gu2024minillm,ko2024distillm}.
GKD matches teacher and student next-token distributions on
student-generated responses \citep{agarwal2024policy}, while MiniLLM minimizes
sequence-level reverse KL using policy gradients \citep{gu2024minillm}.
OPD is used in LLM post-training \citep{yang2025qwen3,lu2025onpolicy}
and has been interpreted as dense KL-constrained RL \citep{yang2026learning}
and online imitation learning \citep{zhang2026does}. However, a stronger
model is not always a better teacher. In classical distillation, students
can learn less from much larger teachers \citep{cho2019efficacy,mirzadeh2020improved}
and fail to match them despite sufficient capacity \citep{stanton2021does}.
In OPD, \citet{li2026rethinking} emphasize alignment between teacher and
student thinking patterns and observe small gradients in failing runs;
\citet{wang2026not} find that teacher supervision is learnable
mainly where the teacher assigns probability mass to the student's top
candidate tokens. Other studies link OPD failures to imbalanced token-level
supervision \citep{fu2026revisiting,jia2026asymmetric},
teacher signals misaligned with task success
\citep{wang2026demystifying,armandpour2026unmasking}, and progressively slower
learning from supervision \citep{fu2026rethinking}. In RL fine-tuning,
low reward variance can lead to small expected gradients
\citep{razin2024vanishing}, and a more accurate reward model need not be a
better teacher \citep{razin2026makes}. These findings motivate our
population-level analysis of how the OPD loss evolves as the student and its
rollout distribution change (Section~\ref{sec:loss-dynamics}).

\paragraph{Training dynamics and geometry of post-training.}
Continuous-time dynamics can approximate gradient descent with small step
sizes \citep{elkabetz2021continuous}, while Polyak--{\L}ojasiewicz-type
conditions yield convergence guarantees for nonconvex optimization,
including policy gradient methods \citep{karimi2016linear,mei2020global}.
Post-training often induces limited model changes. Language-model
fine-tuning can be well described by kernel dynamics
\citep{malladi2023kernel}. Studies of RL post-training report proximity to
the base model in KL divergence \citep{shenfeld2026rl}, updates that
avoid principal weight directions \citep{zhu2025path}, preserved
hidden representations \citep{runwal2026prism}, and many unchanged weights
at BF16 precision \citep{mukherjee2025reinforcement}. In OPD, updates quickly
concentrate in a low-dimensional subspace \citep{shen2026geometry}.
\citet{cai2026learning} explain this early low-rank structure using a local
quadratic model around the initial student with fixed contexts. We study
how teacher proximity shapes OPD loss by expanding around the teacher
under a changing rollout distribution
(Section~\ref{sec:parameter-recovery}) and compare parameter and
representation changes under OPD and RL post-training
(Section~\ref{sec:parameter-representation}).

\section{Preliminaries}
\label{sec:preliminaries}

\subsection{On-policy distillation}
\label{sec:opd-prelim}

\paragraph{Notation.}
A prompt $x\sim\mathcal D$ is followed by a response $y=(y_1,\ldots,y_{|y|})$
of tokens from a vocabulary $\cV$.  The prefix state at position $t$ is
$s_t=(x,y_{<t})$ and the action is the next token $a_t=y_t$, so a policy
$\pi$ factorizes as $\pi(y\mid x)=\prod_{t=1}^{|y|}\pi(a_t\mid s_t)$.  The
student policy $\pi_\theta$ has parameters $\theta$; the teacher policy is
$\pi_T$.  

\paragraph{On-policy distillation.}
OPD minimizes the reverse KL between the student's and the teacher's
next-token distributions on the prefixes of the student's own sampled
responses,
\begin{equation}
 \widehat{\mathcal L}(\theta;\cS)=\frac{1}{|\cS|}
 \sum_{s\in\cS}
 \KL\!\bigl(\pi_\theta(\cdot\mid s)\,\Vert\,\pi_T(\cdot\mid s)\bigr),
 \quad
 \mathcal L(\theta)=\E_{x\sim \mathcal D, y\sim\pi_\theta}\KL\!\bigl(\pi_\theta(\cdot\mid s)\,\Vert\,\pi_T(\cdot\mid s)\bigr),
 \label{eq:online-loss}
\end{equation}
where $\cS$ is a set of prefixes, and a prefix $s=(x,y_{<t})$ is distributed
by drawing $x\sim\mathcal D$, sampling $y\sim\pi_\theta(\cdot\mid x)$, and
taking the position $t$ uniformly from $1,\ldots,|y|$.  In
$\widehat{\mathcal L}$ the set $\cS$ collects the prefixes of the responses
sampled in a minibatch.  

In practice, the sum over the full vocabulary in the
KL is replaced by a top-$K$ version.  At each prefix $s$, we retain only the $K$ tokens with the highest
student probabilities, denoted by $\cV_K(s)$, and renormalize the
student distribution over this set to obtain $\tilde\pi_\theta$.
We then average the log-ratio under this renormalized distribution:
\begin{equation}
  \sum_{a\in\cV_K(s)}\tilde\pi_\theta(a\mid s)
  \bigl[\log\pi_\theta(a\mid s)-\log\pi_T(a\mid s)\bigr].
  \label{eq:topk-kl}
\end{equation}
The full KL needs the teacher's log-probabilities for all $|\cV|$ tokens at
every prefix, whereas Eq~\eqref{eq:topk-kl} needs them for top $K$ tokens and drops
the tail of the distribution. We use a top-$K$ approximation to the reverse KL loss with $K=16$ during training.

\section{Early Loss Plateaus in On-Policy Distillation}
\label{sec:phenomenon}

We compare OPD with different teachers across code generation and
mathematical reasoning. In our experiments, OPD with larger-scale teachers
exhibits early loss plateaus while substantial teacher--student performance
gaps remain. By contrast, self-RL teachers support
sustained loss reduction and larger performance gains.

\begin{table}[htbp]
  \centering
  \normalsize
  \caption{Student--teacher pairs in the seven main OPD runs. Full model names
  and training provenance are given in Appendix~\ref{app:teacher-comparison-details}.}
  \label{tab:student-teacher-pairs}
  \normalsize
  \setlength{\tabcolsep}{5pt}
  \renewcommand{\arraystretch}{1.0}
  \begin{tabular}{@{}lll@{}}
    \hline
    \textbf{Task} & \textbf{Student} & \textbf{Teacher} \\
    \hline
    Code & \QwenFourNT{}
      & \QwenRLCode{} \\
      & & \QwenFourteenNT{} \\
    \hline
    Math & \QwenOneSevenNT{}
      & \QwenRLMath{} \\
      & & \QwenFourNT{} \\
    \hline
    Math & R1-Distill-1.5B
      & JustRL-1.5B \\
      & & \SkyworkTeacher{} \\
      & & R1-Distill-7B \\
    \hline
  \end{tabular}
\end{table}

\subsection{Experimental Setup}
\label{sec:phenomenon-setup}

\paragraph{Models.}
We study three student settings spanning code generation and mathematical
reasoning, with seven student--teacher pairs listed in
Table~\ref{tab:student-teacher-pairs}.  Details of the models are given in
Appendix~\ref{app:teacher-comparison-details}.  We use
\emph{self-RL teachers} to refer specifically to teachers obtained
by further RL training of the initial student, namely \QwenRLCode{} and
JustRL-1.5B, and \emph{larger-scale teachers} to refer to teachers with more
parameters than the student, namely Qwen3-14B, \QwenRLMath{}, Qwen3-4B,
\SkyworkTeacher{}, and R1-Distill-7B.

\paragraph{Datasets.}
For training, we use Eurus-2-RL-Data~\citep{cui2025process} for code generation and DAPO-Math-17K~\citep{yu2026dapo} for mathematical reasoning.

\paragraph{Training.}
We implement OPD in veRL~\citep{sheng2025hybridflow} and optimize the top-$K$ objective in Eq~\eqref{eq:topk-kl} with $K=16$ for the seven main runs. These runs use one rollout per prompt and 64 prompts per update. We train for 200 steps on both code generation and mathematical reasoning. We use AdamW with a constant learning rate of $10^{-6}$, a sampling temperature of 1 for both the student and teacher, and a gradient clipping norm of 1. The remaining hyperparameters are reported in Table~\ref{tab:opd-training-config}.

\paragraph{Evaluation.}
We evaluate code on HumanEval+, MBPP+~\citep{liu2023your}, and LiveCodeBench~\citep{jain2025livecodebench}, and mathematics on AIME 2024, AIME 2025, and AMC 2023. For each benchmark, we sample \(k\) responses per problem and report \(\mathrm{avg}@k\), the fraction of correct responses across all sampled responses and problems, using \(k=4\) for code and \(k=8\) for mathematics. We cap responses at 16,384 tokens for code and 31,744 tokens for mathematics. Task-level validation accuracy is the unweighted mean of the three benchmark accuracies, ensuring that each benchmark contributes equally regardless of its number of problems.

\subsection{Evaluation Metrics}
\label{sec:evaluation-metrics}


\paragraph{Accuracy gain and gap recovery.}
Let $\mathrm{Acc}_{\mathrm{start}}$ be the validation accuracy of the base
student model before any OPD updates (step~0), $\mathrm{Acc}_{\mathrm{end}}$ its accuracy at the
final comparison step, and $\mathrm{Acc}_{\mathrm{teacher}}$ the teacher's
validation accuracy. All accuracies are expressed as percentages (\%). We define the student's accuracy gain as
$\mathrm{Gain}=\mathrm{Acc}_{\mathrm{end}}-\mathrm{Acc}_{\mathrm{start}}$,
reported in percentage points (pp). Gap recovery is the percentage of the
initial teacher--student accuracy gap closed by OPD:
\[
\mathrm{GapRecovery}
=
\frac{\mathrm{Acc}_{\mathrm{end}}-\mathrm{Acc}_{\mathrm{start}}}
     {\mathrm{Acc}_{\mathrm{teacher}}-\mathrm{Acc}_{\mathrm{start}}}
\times 100\%.
\]
\subsection{Experimental Results}
\label{sec:phenomenon-results}

\begin{figure}[!ht]
 \centering
 \includegraphics[width=\textwidth]{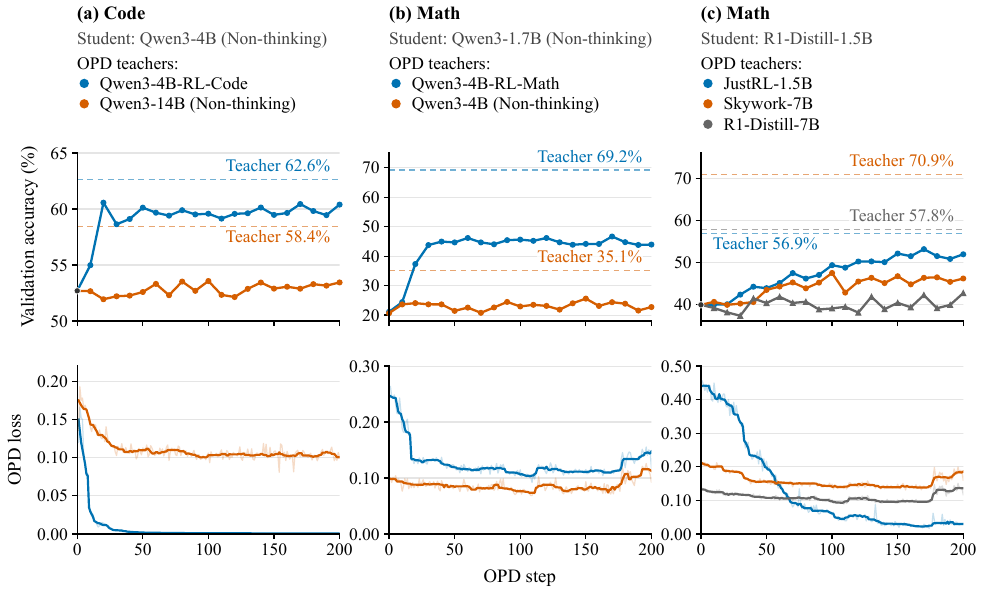}
 \caption{Validation accuracy (top) and OPD loss (bottom) over 200 updates.
 The black point in panels (a) and (c) is the shared initial-student
 evaluation. Dashed lines mark teacher accuracies. Loss curves are centered
 9-step rolling medians; raw losses are shown faintly. Qwen3 models use
 non-thinking mode.}
 \label{fig:teacher-trainability}
\end{figure}

\paragraph{The OPD loss plateaus with larger-scale teachers while the gap remains large.}
\label{sec:loss-plateau}
As Figure~\ref{fig:teacher-trainability} shows, with larger-scale teachers,
the OPD loss stops decreasing early in training, then stays roughly flat in
the code run and rises again in the math runs.
Table~\ref{tab:loss-ability-recovery} quantifies this plateau: the maximum
observed loss reduction in these five runs is only $30.4\%$--$65.1\%$,
compared with $96.3\%$--$99.5\%$ for self-RL teachers. The corresponding
final reductions average $25.1\%$ and $96.2\%$, respectively
(Appendix~\ref{app:teacher-comparison-details}). The
student's validation accuracy plateaus at the same time. After 200 updates, it has
recovered only $13.0\%$--$47.5\%$ of the teacher--student performance gap,
leaving at least half of the gap unclosed.  With a larger-scale teacher, OPD can stop improving the student while much of the accuracy gap remains and the reverse-KL distance to the teacher is still large.

\begin{table}[htbp]
  \centering
  \normalsize
  \setlength{\tabcolsep}{2pt}
  \renewcommand{\arraystretch}{1.5}
  \caption{OPD outcomes over 200 updates: step-1 and minimum loss, loss reduction $(1-\widehat{\mathcal L}_{\min}/\widehat{\mathcal L}_{\mathrm{start}})\times 100\%$ from unrounded losses, accuracy gain from the base student at step~0, and gap recovery (Section~\ref{sec:evaluation-metrics}). Bold marks within-group maxima; $^{\dagger}$ marks self-RL teachers.}
  \label{tab:loss-ability-recovery}
  \begin{tabular}{p{83pt}lrrccc}
    \hline
    \textbf{Student} & \textbf{Teacher} & $\widehat{\mathcal L}_{\mathrm{start}}$ & $\widehat{\mathcal L}_{\min}$ & \textbf{Loss} & \textbf{Gain} & \textbf{Gap} \\
    & & & & \textbf{reduction} & \textbf{(pp)} & \textbf{recovery} \\
    \hline
    \multirow{2}{83pt}{\raggedright \textbf{Code}\\\QwenFourNT{}} & \QwenRLCode{}$^{\dagger}$ & 0.120 & 0.001 & \textbf{99.5\%} & \textbf{+7.70} & \textbf{77.4\%} \\
     & \QwenFourteenNT{} & 0.136 & 0.081 & 40.5\% & +0.75 & 13.0\% \\
    \hline
    \multirow{2}{83pt}{\raggedright \textbf{Math}\\\QwenOneSevenNT{}} & \QwenRLMath{} & 0.263 & 0.092 & \textbf{65.1\%} & \textbf{+22.83} & \textbf{47.5\%} \\
     & \QwenFourNT{} & 0.099 & 0.069 & 30.4\% & +2.07 & 14.3\% \\
    \hline
    \multirow{3}{83pt}{\raggedright \textbf{Math}\\R1-Distill-1.5B\\\strut} & JustRL-1.5B$^{\dagger}$ & 0.455 & 0.017 & \textbf{96.3\%} & \textbf{+12.03} & \textbf{70.7\%} \\
     & \SkyworkTeacher{} & 0.210 & 0.131 & 37.7\% & +6.27 & 20.2\% \\
     & R1-Distill-7B & 0.136 & 0.089 & 34.3\% & +2.80 & 15.6\% \\
    \hline
  \end{tabular}
\end{table}

\paragraph{Self-RL teachers drive the OPD loss to near zero and recover most of the gap.}
\label{sec:teacher-strength}
    Figure~\ref{fig:teacher-trainability} shows that OPD follows a different training dynamic with self-RL teachers than with larger-scale teachers. With self-RL teachers, the OPD loss continues to decrease rather than plateau. As shown in Table~\ref{tab:loss-ability-recovery}, the maximum observed loss reduction reaches 99.5\% for \QwenRLCode{} and 96.3\% for JustRL-1.5B. In the same runs, the student closes 77.4\% and 70.7\% of the accuracy gap to the teacher, respectively.  These results suggest that students can approach self-RL teachers in reverse-KL distance during training and learn transferable abilities from them.

\paragraph{SFT does not resolve the failure, even when followed by OPD.}
\citet{li2026rethinking} use SFT on teacher responses before OPD to recover
failing OPD. In our experiments, SFT alone does no better than direct OPD
with any of the three larger-scale teachers tested. We further continue
OPD for 200 updates from the SFT checkpoints for R1-Distill-7B and
\QwenFourNT{}. The final accuracies are 40.2\% and 22.6\%, respectively,
compared with 42.7\% and 22.7\% for direct OPD. The maximum observed loss
reductions during these subsequent OPD runs are only 37.1\% and 23.0\%.
Thus, SFT initialization does not resolve the limited loss reduction or
improve the final accuracy over direct OPD in these two settings
(Appendix~\ref{app:ablation-sft}).

\paragraph{Ablations of vocabulary support and rollout count.}
For the \QwenOneSevenNT{} math student, ablations comparing top-16 with
full-logit OPD, and one versus eight training rollouts per prompt, show similar
training behavior and closely matched final results
(Appendix~\ref{app:teacher-comparison-details}).

\section{Vanishing Learning Signals: Loss Dynamics of OPD}
\label{sec:main}
Section~\ref{sec:phenomenon-results} raises an optimization question: why does OPD with larger-scale teachers stop reducing its loss, while OPD with self-RL teachers continues to do so?
Section~\ref{sec:loss-dynamics} introduces a continuous-time framework for OPD
and uses it to analyze the loss dynamics.
Section~\ref{sec:learning-signal-collapse} examines gradient-based
diagnostics and their limitations.
Section~\ref{sec:parameter-recovery} gives a sufficient local condition
for loss reduction with teachers sharing the student's parameterization.

\subsection{Mean-Flow Formulation and Loss Dynamics}
\label{sec:loss-dynamics}

Write $\theta_0$ for the student parameters at the start of OPD. By differentiating $\mathcal L(\theta)$, we get:
\begin{equation}
 \nabla \mathcal L(\theta)
 =\underbrace{\E_{\tau\sim\pi_\theta}
       \nabla_\theta\widehat{\mathcal L}(\theta;\tau)}
      _{=:\;g(\theta)}
 +\underbrace{\E_{\tau\sim\pi_\theta}
   \bigl[\widehat{\mathcal L}(\theta;\tau)\,
         \nabla_\theta\log\pi_\theta(y\mid x)\bigr]}
      _{=:\;h(\theta)},
 \label{eq:online-gradient-split}
\end{equation}
Here, the \emph{detached gradient} $g$ is the KL gradient evaluated at prefixes
currently visited by the student, whereas the \emph{occupancy gradient} $h$
captures how changes in $\theta$ alter the prefixes the student visits
\citep{mohamed2020monte}.

At step $m$, OPD updates the parameters according to
$\theta_{m+1}=\theta_m-\eta\,\widehat g_m$, where $\eta$ is the learning rate
and $\widehat g_m$ is an unbiased stochastic gradient of
$\widehat{\mathcal L}$, satisfying
$\mathbb{E}[\widehat g_m\mid\theta_m]=g(\theta_m)$. In the small-learning-rate
limit, the sampling noise averages out. With the rescaled time $t=m\eta$, the
iterates are then described by the ODE
\citep{kushner2003stochastic,borkar2000ode}
\begin{equation}
 \dot\theta(t)=-g(\theta(t)),
 \qquad \theta(0)=\theta_0.
 \label{eq:opd-flow-alignment}
\end{equation}
See Appendix~\ref{app:ode-limit} for the derivation and a discussion of the
idealizations used here.
Under this framework, we obtain the
following exact loss dynamics, together with a factorization of the decay
rate.

\begin{proposition}[OPD loss dynamics]
\label{thm:loss-dynamics}
Under the regularity conditions of Appendix~\ref{app:proof-loss-dynamics},
while the loss is positive, the flow in Eq~\eqref{eq:opd-flow-alignment} satisfies
\begin{equation}
 \frac{d}{dt}\mathcal L(\theta(t))=-2\gamma(\theta(t))\,\mathcal L(\theta(t)),
 \qquad
 \gamma:=\frac{\ip{\nabla\mathcal L}{g}}{2\mathcal L}=\mu(1+\alpha),
 \label{eq:main-flow-identity}
\end{equation}
where $\mu:=\norm{g}^2/(2\mathcal L)$ and $\alpha:=\ip{h}{g}/\norm{g}^2$, the
last equality holding wherever $g\ne0$.
\end{proposition}

The proof is given in Appendix~\ref{app:proof-loss-dynamics}.

Proposition~\ref{thm:loss-dynamics} describes how fast the loss falls.
The quantity $2\gamma$ is the instantaneous fractional rate of loss reduction.
The average of $\gamma$ over $[0,t]$ is
$\overline\gamma(t):=\frac1t\int_0^t\gamma(\theta(u))\,du$.  Integrating Eq.~\eqref{eq:main-flow-identity} gives
$\overline\gamma(t)=\frac{1}{2t}
\log\bigl(\mathcal L(\theta_0)/\mathcal L(\theta(t))\bigr)$,
so $\overline\gamma(t)$ can be computed directly from the loss curve.
A loss plateau, such as those observed in
Section~\ref{sec:phenomenon-results}, corresponds to a small
average of $\gamma$ over the plateau interval.  The
factorization $\gamma=\mu(1+\alpha)$ splits this rate into two factors, so a
small rate can arise in only two ways, through a small $\mu$ or through
$\alpha$ near or below $-1$.

Here, $\mu=\norm{g}^2/(2\mathcal L)$ quantifies the learning signal relative to the
remaining loss; a small $\mu$ indicates a weak signal. The coefficient
$\alpha=\ip{h}{g}/\norm{g}^2$ measures how much of the occupancy gradient $h$
points along the update $g$. So it indicates whether changes in the rollout
distribution reinforce the update ($\alpha>0$) or oppose it ($\alpha<0$).
When $\alpha\leq -1$, the loss no longer decreases, even if $g\neq 0$.

\subsection{Premature Learning-Signal Collapse with Larger-Scale Teachers}
\label{sec:learning-signal-collapse}

Proposition~\ref{thm:loss-dynamics} suggests two possible explanations for
the loss plateau in Section~\ref{sec:phenomenon-results}: the learning signal
may be weak ($\mu$ is small), or the occupancy gradient may oppose the update
($\alpha\le -1$).  We measure both quantities in the code setting with student Qwen3-4B
(Non-thinking) and in
the math setting with student R1-Distill-1.5B.  At every training step $m$, we estimate
$\mu$ from the training logs as
$\widehat\mu_m=\widehat G_m^2/(2\widehat{\mathcal L}_m)$, where $\widehat G_m$
is the gradient norm and $\widehat{\mathcal L}_m$ is the loss.  We estimate $\alpha$ using
$\widehat\alpha_m$ at checkpoints saved every 10 training steps.
At each checkpoint, we use sampled responses to cross-fit
$g$ and $h$.
See Appendix~\ref{app:teacher-comparison-details} for details.
The measured $\widehat\mu_m$ uses a noisy, pre-clipping minibatch gradient;
it does not measure the size of an AdamW parameter update. Moreover, the
squared minibatch gradient norm includes sampling variance, so a decline
in this proxy alone does not establish a decline of the same magnitude
in $\norm{g}^2$.
\paragraph{The logged learning-signal proxy declines while much of the loss remains.}
Figure~\ref{fig:code-signal-collapse} plots the estimated learning signal
$\widehat{\mu}_m$ against the remaining loss at each training step.  With
self-RL teachers, the loss decreases by more than $90\%$ as the
learning signal declines gradually.  By contrast, with larger-scale teachers, the logged proxy drops by approximately an order of magnitude while over
$60\%$ of the loss remains. After this
drop, the loss changes little.  We call this pattern \emph{premature
learning-signal collapse}: the logged gradient weakens before the teacher--student
mismatch has been substantially reduced. Here, ``premature'' describes the
early timing and remaining disagreement; it does not establish that the
student can achieve a lower loss under the same parameterization.

\begin{figure}[tbp]
 \centering
 \includegraphics[width=\textwidth]{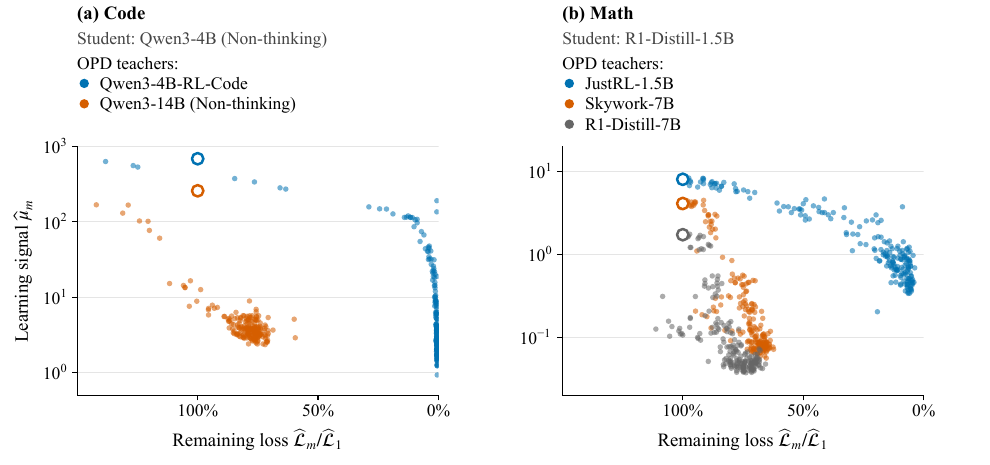}
 \caption{Learning signal $\widehat\mu_m$ versus remaining loss
 $\widehat{\mathcal L}_m/\widehat{\mathcal L}_1$ over 200 updates.
 Remaining loss decreases from left to right; the learning-signal axis is
 logarithmic. Each dot is one training step; open circles mark step~1.
 Colors and teachers match Figure~\ref{fig:teacher-trainability}.
 Appendix~\ref{app:teacher-comparison-details} gives the estimators and
 sources, and shows the corresponding time series in
 Figure~\ref{fig:signal-collapse-timeseries}.}
 \label{fig:code-signal-collapse}
\end{figure}

\paragraph{The occupancy gradient does not consistently explain the loss plateaus.}
Table~\ref{tab:occupancy-factor} shows that, in the code setting,
$1+\widehat\alpha_m$ is positive at 19 of 20 checkpoints for Qwen3-14B.  These
estimates provide little evidence for persistent occupancy cancellation in
this run.  In the math setting, $1+\widehat{\alpha}_m$ is positive at 12 of the 20
checkpoints for \SkyworkTeacher{} and 13 of the 20 checkpoints for
R1-Distill-7B. The other estimates are consistent with the occupancy gradient offsetting
or reversing loss reduction in the idealized flow. This may slow training further, alongside the collapse
of $\mu$.  

Across the code and math settings, an early drop in $\widehat\mu_m$ is the
common pattern. Cross-fitted estimates of $\norm{g}^2$ decline from early to
middle training in Math but remain roughly stable for Qwen3-14B after
step 20. The occupancy estimates remain tentative, however, because
$\alpha$ is measured at only 20 checkpoints per teacher in each setting.

\begin{table}[htbp]
  \centering
  \normalsize
  \caption{Occupancy factor $1+\widehat\alpha_m$ for larger-scale teachers:
  median, interquartile range (IQR), and counts above zero across 20
  checkpoints (steps 10--200, every 10 steps). Quartiles use linear
  interpolation. The estimator is given in
  Appendix~\ref{app:teacher-comparison-details}.}
  \label{tab:occupancy-factor}
  \begin{tabular}{@{}llccc@{}}
    \toprule
    & & \multicolumn{3}{c}{$1+\widehat\alpha_m$} \\
    \cmidrule(l){3-5}
    \textbf{Task} & \textbf{Teacher} & \textbf{Median} & \textbf{IQR} & \textbf{Positive} \\
    \midrule
    Code & Qwen3-14B & $1.10$ & $[0.67,\,1.38]$ & 19/20 \\
    \midrule
    Math & \SkyworkTeacher{} & $0.13$ & $[-2.28,\,1.48]$ & 12/20 \\
         & R1-Distill-7B & $0.41$ & $[-0.83,\,1.37]$ & 13/20 \\
    \bottomrule
  \end{tabular}
\end{table}

\subsection{Local Convergence with Self-RL Teachers}
\label{sec:parameter-recovery}

We now study a sufficient local condition for continued loss reduction
that is motivated by the self-RL teacher experiments.  Since each such teacher is obtained by further RL training of the
student (Section~\ref{sec:phenomenon-setup}), it shares the parameter space of
the student, and we write it as $\theta_T=\theta_0+\Delta$.  We also find that
the teacher is close to the student in this space.  As
Table~\ref{tab:parameter-representation} shows, the relative distance
$\norm\Delta/\norm{\theta_0}$ is 0.030\% for \QwenRLCode{} and 0.50\% for
JustRL-1.5B.
As a supplementary run, we also apply OPD with the self-RL
instruction-following teacher UltraData-IF-1.5B, whose distance is 0.037\%
(Appendix~\ref{app:if-experiment}).
If the teacher is close enough, the OPD loss around it is approximately
quadratic (Proposition~\ref{prop:local-expansion}).  On such a quadratic, the
OPD flow keeps reducing the loss.  Theorem~\ref{thm:main} makes this precise
and shows how the remaining loss depends on the distance.
\paragraph{Setting.}
We analyze OPD in the continuous-time ODE framework of
Eq~\eqref{eq:opd-flow-alignment}.  We also use three regularity conditions that
are standard in policy-gradient analysis, namely (i) bounded rollout length
\citep{razin2026makes}, (ii) a bounded and Lipschitz policy score
\citep{fatkhullin2023stochastic,barakat2022analysis,chakraborty2024parl}, and
(iii) a Lipschitz Hessian of $\mathcal L$ \citep{jaiswal2024low}.
Appendix~\ref{app:main-proof} states them in full.  Our OPD experiments with self-RL teachers in
Section~\ref{sec:phenomenon-results} motivate this analysis but do not
directly test its local, spectrum-dependent bound.

\paragraph{Slow directions and residual loss.}
The recovery bound depends on both the teacher--student distance
$\rho=\|\Delta\|$ and the directions of $\Delta$. We characterize these
directions using the initial student's Fisher matrix $H_0=H(\theta_0)$,
which locally approximates the loss Hessian for nearby teachers
(Appendix~\ref{app:endpoint-geometry}). In the quadratic approximation,
loss along an eigenvector $u_i$ of $H_0$ with eigenvalue $\lambda_i$ decays as $e^{-2\lambda_i t}$, so directions
with small eigenvalues are learned slowly (Appendix~\ref{app:local}).
For a cutoff $\delta>0$, their contribution to the initial loss is
\begin{equation}
E_\Delta(\delta)
:=\frac12\sum_{0<\lambda_i\le\delta}
\lambda_i(u_i^\top\Delta)^2
\le \frac{\delta}{2}\rho^2.
\label{eq:endpoint-tail}
\end{equation}
\begin{theorem}[Abridged version of Theorem~\ref{thm:detailed-recovery}]
\label{thm:main}
Let $\theta_T$ be a self-RL teacher, with $\rho=\norm\Delta>0$.  Under the conditions of Appendix~\ref{app:main-proof}, the OPD flow in
Eq~\eqref{eq:opd-flow-alignment} stays within distance $2\rho$ of the teacher up
to a time $T_{\mathrm{loc}}=\Omega(1/\rho)$.  For every $\delta>0$ and
$0\le t\le T_{\mathrm{loc}}$,
\begin{equation}
 \mathcal L(\theta(t))
 \le e^{-\delta t}\mathcal L(\theta_0)+E_\Delta(\delta)
 +O\bigl(\rho^3+\rho^4/\delta\bigr),
 \label{eq:main-spectral-rate}
\end{equation}
where $\Omega(\cdot)$ and $O(\cdot)$ hide constants depending on the policy-score
bounds, the rollout length and the Lipschitz constant of the Hessian of
$\mathcal L$.
\end{theorem}

The proof is given in Appendix~\ref{app:main-proof}.
Theorem~\ref{thm:main} bounds the loss by an exponentially decaying term plus
a residual determined by $E_\Delta(\delta)$ and higher-order errors in $\rho$.  This is consistent with
Section~\ref{sec:phenomenon-results}, where the minimum loss reaches 0.5\% and
3.7\% of its starting value with the two self-RL teachers.

We emphasize that Theorem~\ref{thm:main} does not suggest that only
self-RL teachers are easy to learn from.  Its proof does not use how
the teacher was trained, so OPD can reach any teacher in the student's parameter
space that is close enough to the student and differs from it mostly along
high-curvature directions.  Rather, it
explains why the self-RL teachers in our experiments, which lie close
to their students, can be reached by OPD, as shown in
Section~\ref{sec:phenomenon-results}.  The larger-scale teachers are not
covered, because their parameter dimensions differ from those of their
students.

\section{Parameter and Representation Changes in OPD}
\label{sec:parameter-representation}

Section~\ref{sec:main} shows that OPD can reach a teacher close to the
student, but its learning signal collapses early with larger-scale teachers.
A possible reason is that OPD only makes local changes to the student, which
may not suffice when matching the teacher requires changes these updates
do not make.  We
therefore measure how much OPD changes the student, both in its parameters and
in its hidden representations.

\paragraph{Metrics.}
We use the relative parameter change
$d_{\mathrm{par}}=\norm{\theta_{\mathrm{end}}-\theta_{\mathrm{start}}}/\norm{\theta_{\mathrm{start}}}$
to measure how far OPD moves the student in parameter space.  Because parameter
distances say little about what the model computes, we also use linear Centered Kernel Alignment (CKA)
\citep{kornblith2019similarity} to measure how much its hidden representations
change.  For a set of prompts, let $X$ and $Y$ be the hidden states of the
starting and final checkpoints at one layer, with one row per prompt, averaged
over its tokens and centered across prompts.  Then
$\operatorname{CKA}(X,Y)=\norm{Y^\top X}_F^2/(\norm{X^\top X}_F\,\norm{Y^\top Y}_F)$,
which equals 1 when the two representations agree up to rotation and scaling.
We use the prompts of the evaluation benchmarks, namely AIME24, AIME25 and
AMC23 in the math setting (143 prompts) and HumanEval+, MBPP+ and
LiveCodeBench in the code setting (673 prompts).  See
Appendix~\ref{app:representation-details} for details.

\paragraph{OPD changes the student's parameters and representations only slightly.}
Table~\ref{tab:parameter-representation} reports both metrics for eight OPD
runs after 200 updates.  In every run, $d_{\mathrm{par}}$ is between 0.025\% and
0.098\%, and CKA is at least 0.985 at every layer.  Under this learning rate and training budget, the measured changes in
parameters and representations are small. The displacement magnitude alone does not separate
successful from stalled runs. Whether the direction of these small changes
limits teacher matching remains a hypothesis for controlled interventions.

\begin{table}[htbp]
 \centering
 \caption{Parameter and representation changes from the student's initialization.
 \emph{OPD student} is measured after 200 updates; \emph{Teacher} compares
 the self-RL teacher ($^{\dagger}$). For teachers with a different parameterization,
 $d_{\mathrm{par}}$ is undefined and teacher--student CKA was not computed (--). $d_{\mathrm{par}}$ denotes relative parameter change;
 Min.\ CKA is the minimum layerwise CKA. IF details:
 Appendix~\ref{app:if-experiment}.}
 \label{tab:parameter-representation}
 \normalsize
 \setlength{\tabcolsep}{4pt}
 \begin{tabular}{@{}lllcccc@{}}
\toprule
 & & & \multicolumn{2}{c}{\textbf{OPD student}} & \multicolumn{2}{c}{\textbf{Teacher}} \\
\cmidrule(lr){4-5}\cmidrule(l){6-7}
\textbf{Task} & \textbf{Student} & \textbf{Teacher}
 & $d_{\mathrm{par}}$ & Min.\ CKA & $d_{\mathrm{par}}$ & Min.\ CKA \\
\midrule
Code & Qwen3-4B & \QwenRLCode{}$^{\dagger}$ & 0.032\% & 0.9979 & 0.030\% & 0.9993 \\
     &          & Qwen3-14B                 & 0.063\% & 0.9850 & --      & --     \\
\midrule
Math & Qwen3-1.7B & \QwenRLMath{}          & 0.031\% & 0.9939 & --      & --     \\
     &            & Qwen3-4B                  & 0.025\% & 0.9973 & --      & --     \\
\midrule
Math & R1-Distill-1.5B & JustRL-1.5B$^{\dagger}$       & 0.098\% & 0.9904 & 0.497\% & 0.9674 \\
     &                 & \SkyworkTeacher{}             & 0.063\% & 0.9988 & --      & --     \\
     &                 & R1-Distill-7B                 & 0.050\% & 0.9986 & --      & --     \\
\midrule
IF   & R1-Distill-1.5B & UltraData-IF-1.5B$^{\dagger}$ & 0.026\% & 0.9999 & 0.037\% & 0.9998 \\
\bottomrule
\end{tabular}

\end{table}

\paragraph{Self-RL teachers remain close to the initial student.}
The Teacher columns of Table~\ref{tab:parameter-representation} compare each
self-RL teacher with the base model it was trained from.  RL
post-training moves the base model by 0.030\% to 0.50\%, and its minimum CKA is
between 0.967 and 0.9998.  These changes are comparable to or moderately
larger than those made by OPD (up to about $5\times$, for JustRL).  These measurements motivate the local recovery analysis of
Section~\ref{sec:parameter-recovery}, but small relative distances alone
do not verify its assumptions on absolute distance and curvature.

\section{Conclusion}
\label{sec:conclusion}

We study why OPD can plateau while substantial teacher--student
disagreement remains. Across code generation and mathematical reasoning,
our training logs show an early decline in a gradient-based learning-signal
proxy while substantial loss remains, whereas self-RL teachers support
sustained loss reduction. Our continuous-time analysis separates the effects of the
learning signal and the changing rollout distribution, and provides a
local recovery guarantee for sufficiently nearby teachers under regularity
conditions. Small parameter changes and largely preserved representations
further suggest that limited adaptation may constrain learning from other
teachers, although displacement magnitude alone does not separate
successful from stalled runs.
Our local guarantee assumes an idealized gradient flow and shared
parameterization, while the experiments use AdamW and a noisy gradient
proxy; connecting these diagnostics to actual updates remains future work.
Controlled interventions across learning rates, training budgets, and
seeds are also needed to distinguish optimization limitations from
representational limits and to test the proposed role of limited adaptation.

\clearpage
\bibliographystyle{ims}
\bibliography{references}

\clearpage
\appendix

\section{OPD calculus and continuous-time identity}
\label{app:opd-calculus}

\subsection{Per-prefix reverse-KL gradient}

Let
\begin{equation}
 r_\theta(a,s)
 :=\log\frac{\pi_\theta(a\mid s)}{\pi_T(a\mid s)}.
 \label{eq:log-ratio}
\end{equation}
Write $\ell(\theta;s)=\KL\!\bigl(\pi_\theta(\cdot\mid s)\,\Vert\,\pi_T(\cdot\mid s)\bigr)$
for the per-prefix loss in Eq~\eqref{eq:online-loss}.  At a fixed prefix,
\begin{align}
 \nabla_\theta\ell(\theta;s)
 &=\sum_a\pi_\theta(a\mid s)
   \bigl(r_\theta(a,s)+1\bigr)\psi_\theta(a,s)\\
 &=\E_{a\sim\pi_\theta(\cdot\mid s)}
 \left[
  \left(r_\theta(a,s)-
  \E_{b\sim\pi_\theta}r_\theta(b,s)\right)
  \psi_\theta(a,s)
 \right].
 \label{eq:centered-reverse-gradient}
\end{align}
The second equality uses
$\E_{a\sim\pi_\theta}\psi_\theta(a,s)=0$; the centered baseline is optional.
For the forward KL
$\KL(\pi_T(\cdot\mid s)\Vert\pi_\theta(\cdot\mid s))$, the corresponding
fixed-prefix gradient is
$-\E_{a\sim\pi_T}\psi_\theta(a,s)$.

\subsection{Population gradient split}

Assume differentiation may be interchanged with expectation in
Eq~\eqref{eq:online-loss}.  Since the prompt distribution does not depend on
$\theta$, the score of a trajectory $\tau=(x,y)$ is
$\nabla_\theta\log\pi_\theta(y\mid x)$, and the likelihood-ratio identity gives
\begin{align}
 \nabla \mathcal L(\theta)
 &=\E_{\tau\sim\pi_\theta}
       \nabla_\theta\widehat{\mathcal L}(\theta;\tau)
 +\E_{\tau\sim\pi_\theta}
 \left[
  \widehat{\mathcal L}(\theta;\tau)\nabla_\theta\log\pi_\theta(y\mid x)
 \right]\\
 &=g(\theta)+h(\theta),
\end{align}
which proves Eq~\eqref{eq:online-gradient-split}.  Write
$\cA(\theta):=\ip{\nabla \mathcal L(\theta)}{g(\theta)}$
for the loss--update alignment, so that $\cA=2\gamma\mathcal L$ in the
notation of Proposition~\ref{thm:loss-dynamics}.  Substituting the split
gives
\begin{equation}
 \cA(\theta)
 =\norm{g(\theta)}^2
 +\ip{h(\theta)}{g(\theta)}.
 \label{eq:alignment-expanded}
\end{equation}

\subsection{From the OPD update to the ODE}
\label{app:ode-limit}

Write the update of Section~\ref{sec:loss-dynamics} as
\begin{equation}
 \theta_{m+1}=\theta_m-\eta\,\widehat g_m,
 \qquad
 \widehat g_m=g(\theta_m)+\xi_m,
 \label{eq:opd-sgd}
\end{equation}
where $\xi_m$ is the sampling noise of the minibatch gradient, with
$\E[\xi_m\mid\theta_m]=0$.  Take a short interval of length $\Delta t$ that
contains $n=\Delta t/\eta$ steps starting from step $m$.  Summing
Eq~\eqref{eq:opd-sgd} over these steps gives
\[
 \theta_{m+n}-\theta_m
 =-\eta\sum_{k=m}^{m+n-1}g(\theta_k)-\eta\sum_{k=m}^{m+n-1}\xi_k .
\]
Because the iterates move little over the interval and $g$ is continuous,
the first sum equals $\Delta t\,g(\theta_m)$ up to a term of order
$\Delta t^2$.  The second sum has mean zero and, for noise with bounded
variance, a standard deviation of order $\sqrt{\eta\,\Delta t}$, which is
negligible next to $\Delta t$ when $\eta\to0$.  Writing
$\theta(t)=\theta_{t/\eta}$, dividing by $\Delta t$, and letting first
$\eta\to0$ and then $\Delta t\to0$ gives Eq~\eqref{eq:opd-flow-alignment}.
This is the standard ODE limit of a stochastic iteration
\citep{kushner2003stochastic,borkar2000ode}, and
Proposition~\ref{prop:finite-step} below bounds the loss after finitely many
steps without taking the limit.

Our runs use AdamW, whereas Eq~\eqref{eq:opd-flow-alignment} idealizes
detached-gradient descent. AdamW's momentum, preconditioning, and weight decay
can change the update direction; gradient clipping also rescales the gradient.  \citet{razin2026makes} make
the same idealization in their continuous-time analysis of policy gradient.
Finally, the analysis uses the full-vocabulary KL in Eq~\eqref{eq:online-loss},
whereas our main runs optimize its top-16 version in Eq~\eqref{eq:topk-kl}.  The
statistics measured in Section~\ref{sec:learning-signal-collapse} are computed
on the top-16 objective and serve as proxies for the corresponding full-KL
quantities.

\subsection{\texorpdfstring{Proof of Proposition~\ref{thm:loss-dynamics}}{Proof of the loss-dynamics proposition}}
\label{app:proof-loss-dynamics}

\paragraph{Regularity conditions.}
Proposition~\ref{thm:loss-dynamics} uses two conditions.
\emph{Differentiation:} differentiation and expectation may be interchanged
in Eq~\eqref{eq:online-loss}, so that the split in Eq~\eqref{eq:online-gradient-split}
holds.
\emph{Flow regularity:} there is an open set $\cU\ni\theta_0$ on which
$\mathcal L$ is $C^1$ and $g$ is locally Lipschitz, and
Eq~\eqref{eq:opd-flow-alignment} has a solution that stays in $\cU$ on
$[0,T_\star]$ for the horizon $T_\star$ under consideration.  Since $g$ is
locally Lipschitz, this solution is unique and $C^1$ (Picard--Lindel\"of), so
$t\mapsto\mathcal L(\theta(t))$ is $C^1$ and the chain rule applies.

\begin{proof}[Proof of Proposition~\ref{thm:loss-dynamics}]
\emph{Rate identity in Eq~\eqref{eq:main-flow-identity}.}
By the chain rule and Eq~\eqref{eq:opd-flow-alignment}, for $t\in[0,T_\star]$,
\[
 \frac{d}{dt}\mathcal L(\theta(t))
 =\ip{\nabla\mathcal L(\theta(t))}{\dot\theta(t)}
 =-\ip{\nabla\mathcal L(\theta(t))}{g(\theta(t))}
 =-\cA(\theta(t)).
\]
Wherever $\mathcal L(\theta(t))>0$, the definition $\gamma=\cA/(2\mathcal L)$
rewrites the right-hand side as $-2\gamma(\theta(t))\,\mathcal L(\theta(t))$,
which is Eq~\eqref{eq:main-flow-identity}.

\emph{Integrated form.}
Let $t_0:=\inf\{t\in[0,T_\star]:\mathcal L(\theta(t))=0\}$, with
$t_0:=T_\star$ if the loss never vanishes.  On $[0,t_0)$ the loss is positive,
so $\log\mathcal L(\theta(t))$ is $C^1$ with
$\frac{d}{dt}\log\mathcal L(\theta(t))=-2\gamma(\theta(t))$ by
Eq~\eqref{eq:main-flow-identity}.  Integrating from $0$ to $t<t_0$ and
exponentiating gives
\begin{equation}
 \mathcal L(\theta(t))
 =\mathcal L(\theta_0)
 \exp\!\left(-2\int_0^t\gamma(\theta(u))\,du\right),
 \label{eq:main-exact-rate}
\end{equation}
which yields the run average $\overline\gamma(t)$ of
Section~\ref{sec:loss-dynamics}.

\emph{Factorization.}
Wherever $\mathcal L(\theta)>0$ and $g(\theta)\ne0$, dividing
Eq~\eqref{eq:alignment-expanded} by $2\mathcal L(\theta)$ gives
\[
 \gamma
 =\frac{\norm{g}^2}{2\mathcal L}
 +\frac{\ip{h}{g}}{2\mathcal L}
 =\frac{\norm{g}^2}{2\mathcal L}
  \left(1+\frac{\ip{h}{g}}{\norm{g}^2}\right)
 =\mu(1+\alpha)
\]
by the definitions of $\mu$ and $\alpha$.
\end{proof}

By Cauchy--Schwarz, $|\alpha|\le\norm{h}/\norm{g}$, so the factorization
gives $\mu(1-\norm{h}/\norm{g})\le\gamma\le\mu(1+\norm{h}/\norm{g})$.

\begin{proposition}[Finite-step stochastic decomposition]
\label{prop:finite-step}
Suppose $\mathcal L$ is $\beta$-smooth on a region containing the iterates,
$\E[\widehat g_m\mid\cF_m]=g(\theta_m)$, and
\begin{equation}
 \sigma_m^2
 :=\E\left[
  \norm{\widehat g_m-g(\theta_m)}^2
  \mid\cF_m
 \right]<\infty.
\end{equation}
Then
\begin{equation}
 \begin{split}
 \E \mathcal L(\theta_M)
 \le{}&\mathcal L(\theta_0)
 -\sum_{m=0}^{M-1}\eta_m\E\cA(\theta_m)\\
 &+\frac{\beta}{2}\sum_{m=0}^{M-1}\eta_m^2
 \E\left[\norm{g(\theta_m)}^2+\sigma_m^2\right].
 \end{split}
 \label{eq:finite-step}
\end{equation}
\end{proposition}

\begin{proof}
Smoothness and the update $\theta_{m+1}=\theta_m-\eta_m\widehat g_m$ imply
\[
 \E_m\mathcal L(\theta_{m+1})
 \le \mathcal L(\theta_m)-\eta_m\cA(\theta_m)
 +\frac{\beta\eta_m^2}{2}
  \left(\norm{g(\theta_m)}^2+\sigma_m^2\right).
\]
Take total expectations and sum over $m$.
\end{proof}

\section{Parameter differences and the initial Fisher}
\label{app:endpoint-geometry}

Throughout this section, the teacher and student use the same architecture,
parameter coordinates, tokenizer, vocabulary, temperature, and policy
interface. Set $\theta_T=\theta_{\mathrm{RL}}$ and
$\Delta=\theta_T-\theta_0$. All expectations below are assumed finite,
with $N(\tau)>0$ almost surely. For a policy checkpoint $\vartheta$, write
$\psi_\vartheta(a,s)=\nabla_\vartheta\log\pi_\vartheta(a\mid s)$ and define
\begin{equation}
 H(\vartheta):=\E_{\tau\sim\pi_\vartheta}
 \left[
  \frac{1}{N(\tau)}\sum_{s\in\cS(\tau)}
  \E_{a\sim\pi_\vartheta(\cdot\mid s)}
  \psi_\vartheta(a,s)\psi_\vartheta(a,s)^\top
 \right].
 \label{eq:self-fisher}
\end{equation}
Thus $H_0=H(\theta_0)$. The teacher Fisher and its difference from the
initial-student Fisher are
\begin{equation}
 H_T:=H(\theta_T),
 \qquad
 \varepsilon_H:=\norm{H_T-H_0}_{\mathrm{op}}.
 \label{eq:teacher-fisher}
\end{equation}
Both matrices use the expectation of the token-normalized statistic, not
a ratio of separate expectations of its numerator and denominator.

\begin{lemma}[Low-curvature endpoint energy]
\label{lem:endpoint-tail}
For $H_0\succeq0$, let $P_{\le\delta}$ be its spectral projector onto
eigenvalues in $[0,\delta]$. Then
\begin{equation}
 E_\Delta(\delta)
 =\tfrac12\Delta^\top H_0P_{\le\delta}\Delta,
 \qquad
 0\le E_\Delta(\delta)\le\tfrac\delta2\norm\Delta^2.
 \label{eq:endpoint-tail-bound}
\end{equation}
No restriction of $\Delta$ to the range of $H_0$ is needed.
\end{lemma}

\begin{proof}
Write $d_i=u_i^\top\Delta$ in an orthonormal eigenbasis of $H_0$.
Expanding the quadratic form gives
$E_\Delta(\delta)=\tfrac12\sum_{0<\lambda_i\le\delta}\lambda_i d_i^2$.
Each retained eigenvalue is at most $\delta$, so the sum is bounded by
$\delta\sum_i d_i^2/2=\delta\norm\Delta^2/2$.
Zero eigenvalues contribute no energy; nullspace components need not vanish.
\end{proof}

\begin{lemma}[Same-model root and Fisher identities]
\label{lem:root-fisher}
Assume positive conditional token probabilities and a rollout law with
common support in a neighborhood of $\theta_T$. Assume the smoothness and
integrable domination needed to differentiate $\mathcal L$ twice and
$g$ once. Then
\begin{equation}
 \mathcal L(\theta_T)=0,
 \qquad \nabla \mathcal L(\theta_T)=0,
 \qquad g(\theta_T)=0,
 \label{eq:root}
\end{equation}
and
\begin{equation}
 \nabla^2\mathcal L(\theta_T)
 =Dg(\theta_T)=H_T.
 \label{eq:root-fisher}
\end{equation}
\end{lemma}

\begin{proof}
At $\theta_T$, every conditional KL and its first parameter derivative
vanish for every fixed prefix. Its Hessian is the conditional policy
Fisher. Hence $\widehat{\mathcal L}(\theta_T;\tau)$ and its first parameter
derivative vanish for every sampled trajectory $\tau$.
When differentiating $\mathcal L(\theta)=\E_{\tau\sim\pi_\theta}\widehat{\mathcal L}(\theta;\tau)$ twice, all terms
involving derivatives of the rollout law multiply one of these vanishing
quantities. The remaining term is
$\E_{\tau\sim\pi_{\theta_T}}\nabla_\theta^2\widehat{\mathcal L}(\theta_T;\tau)=H_T$.
Differentiating $g$ once gives the same term because the
derivative of the rollout law multiplies
$\nabla_\theta\widehat{\mathcal L}(\theta_T;\tau)=0$.
The denominator $N(\tau)$ is held fixed conditional on $\tau$; its
distributional dependence is already included in the law of $\tau\sim\pi_\theta$.
\end{proof}

\paragraph{Score and sampling conditions.}
We impose boundedness and Lipschitz continuity of the policy score,
\begin{equation}
 \norm{\psi_\theta(a,s)}\le B,
 \qquad
 \norm{\psi_\theta(a,s)-\psi_{\theta'}(a,s)}
 \le L_s\norm{\theta-\theta'},
 \label{eq:policy-score-regularity}
\end{equation}
uniformly over the relevant prefixes and tokens and over parameters near
$\theta_T$.  These conditions are standard in policy-gradient analysis
\citep[Assumption~1]{fatkhullin2023stochastic} and are also used by
\citet[Assumption~3.1]{barakat2022analysis} and in RLHF
\citep[Assumption~3, Eq~(19)]{chakraborty2024parl}.
Prompts follow a fixed distribution.  Each sampled trajectory uses at most
$M$ autoregressive token draws from the full student policy, with fixed
stopping and masking rules and $N(\tau)=|\cS(\tau)|>0$.
Here $M$ bounds all draws in $\tau$, including every response when $\tau$
collects a batch.  A finite generation horizon is also used by
\citet[Section~3.1]{razin2026makes}.
Lemma~\ref{lem:score-update} derives a quadratic update remainder with
\begin{equation}
 C_g:=\tfrac32B L_s+(M+1)B^3.
 \label{eq:score-remainder-constant}
\end{equation}

\begin{lemma}[Policy scores and the detached update]
\label{lem:score-update}
Let $\mathcal B_R$ be a closed ball of radius $R$ about $\theta_T$.
Assume positive differentiable token probabilities and the uniform score
bounds in Eq~\eqref{eq:policy-score-regularity} on an open neighborhood of this
ball. Prompts and auxiliary randomness have parameter-independent laws.
Each trajectory is generated using at most $M$ draws from $\pi_\theta$;
stopping, masking, and postprocessing are fixed functions of the generated
trajectory and auxiliary randomness. Suppose
$N(\tau)=|\cS(\tau)|>0$ almost surely.
Then $g$ is locally Lipschitz near the ball and
\begin{equation}
 \norm{g(\theta_T+e)-H_Te}
 \le C_g\norm e^2,
 \qquad \theta_T+e\in\mathcal B_R,
 \label{eq:score-update-remainder}
\end{equation}
with $C_g$ from Eq~\eqref{eq:score-remainder-constant}.
The self-Fisher map also satisfies
\begin{equation}
 \norm{H(\theta)-H(\theta')}_{\mathrm{op}}
 \le L_H\norm{\theta-\theta'},
 \qquad L_H:=2B L_s+(M+1)B^3,
 \label{eq:score-fisher-stability}
\end{equation}
for $\theta,\theta'\in\mathcal B_R$.
\end{lemma}

\begin{proof}
Write $d=\norm e$, $p_\theta(a)=\pi_\theta(a\mid s)$, and
$b_s(\theta)=\nabla_\theta\ell(\theta;s)$ for a fixed prefix $s$.
The score identity $\sum_a p_\theta(a)\psi_\theta(a,s)=0$ gives
\[
 b_s(\theta)=\sum_a p_\theta(a)\psi_\theta(a,s)
       \log\frac{p_\theta(a)}{p_T(a)}.
\]
Integrating along the segment from $\theta_T$ to $\theta_T+e$ yields
\begin{align}
 \sum_a|p_\theta(a)-p_T(a)|&\le Gd,\\
 \left|\log\frac{p_\theta(a)}{p_T(a)}\right|&\le Gd,\\
 \left|\log\frac{p_\theta(a)}{p_T(a)}
             -\psi_T(a,s)^\top e\right|&\le\tfrac{L_s}{2}d^2.
 \label{eq:score-log-ratio-bounds}
\end{align}
Let $H_T(s)=\sum_a p_T(a)\psi_T(a,s)\psi_T(a,s)^\top$.
Subtracting $H_T(s)e$ from $b_s(\theta)$ gives three terms: the change
in the score, the remainder of the log-ratio, and the change from $p_T$
to $p_\theta$. Their norms are bounded by $B L_sd^2$,
$B L_sd^2/2$, and $B^3d^2$, respectively. Hence
\begin{equation}
 \norm{b_s(\theta)-H_T(s)e}
 \le C_{\mathrm{pre}}d^2,
 \qquad C_{\mathrm{pre}}:=\tfrac32B L_s+B^3.
 \label{eq:prefix-score-remainder}
\end{equation}

We next control the trajectory distribution.  In this proof, write $Q_\theta$
for the law of the recorded trajectory $\tau\sim\pi_\theta$.  The likelihood score of the
complete generated trajectory, including auxiliary randomness before
postprocessing, is a sum of at most $M$ token scores and has
norm at most $MG$. Stopped sequences may be padded by absorbing states
with zero score. Integrating the derivative of its law along the same
segment, then applying the fixed postprocessing map, gives
\begin{equation}
 \norm{Q_\theta-Q_{\theta_T}}_1\le MGd,
 \qquad \norm{Q-Q'}_1:=\int|q-q'|.
 \label{eq:record-law-stability}
\end{equation}
For $A_T(\tau):=N(\tau)^{-1}\sum_{s\in\cS(\tau)}H_T(s)$,
we have $\norm{A_T(\tau)}_{\mathrm{op}}\le B^2$ and
$\E_{\tau\sim\pi_{\theta_T}}A_T(\tau)=H_T$.
Averaging Eq~\eqref{eq:prefix-score-remainder} over scored prefixes and
then over $\tau\sim\pi_\theta$ therefore gives
\[
 \norm{g(\theta)-H_Te}
 \le C_{\mathrm{pre}}d^2
    +B^2d\norm{Q_\theta-Q_{\theta_T}}_1
 \le C_gd^2.
\]
This averages each trajectory's normalized statistic; the random denominator
is not replaced by its expectation.

For any two parameters in the ball, the conditional-policy law changes
in $L^1$ by at most $B\norm{\theta-\theta'}$.
The outer product of two scores changes in operator norm by at most
$2B L_s\norm{\theta-\theta'}$.
Combining these bounds with the trajectory-law bound gives
Eq~\eqref{eq:score-fisher-stability}.
Finally, the scores and log-ratios are bounded and Lipschitz on the ball.
Applying the same distributional bounds to their product in $b_s$ gives
\[
 \norm{g(\theta)-g(\theta')}
 \le\bigl[B^2+B L_sR+(M+1)B^3R\bigr]\norm{\theta-\theta'}.
\]
The argument also applies on a slightly larger ball inside the open
regularity neighborhood, proving the stated local Lipschitz property.
\end{proof}

The bounded-draw mechanism is part of this lemma's scope: conditioning on
accepted trajectories or resampling empty responses requires controlling the
resulting law separately unless all retries obey the stated draw bound.
Parameter-dependent top-$k$ or top-$p$ truncation is not included.

\begin{proposition}[Local approximation from endpoint geometry]
\label{prop:endpoint-stability}
Assume Lemma~\ref{lem:root-fisher} and the score and sampling conditions
of Lemma~\ref{lem:score-update} on a neighborhood of a ball
$\mathcal B_R$ containing $\theta_0$.
Suppose $\mathcal L$ is $C^2$ with an $L_{\mathcal L}$-Lipschitz Hessian there, measured in
operator norm. For $e=\theta-\theta_T$ and $\theta\in\mathcal B_R$,
\begin{align}
 \left|\mathcal L(\theta)-\tfrac12e^\top H_0e\right|
 &\le\tfrac{\varepsilon_H}{2}\norm e^2
       +\tfrac{L_{\mathcal L}}{6}\norm e^3,\\
 \norm{g(\theta)-H_0e}
 &\le\varepsilon_H\norm e+C_g\norm e^2.
 \label{eq:endpoint-taylor-errors}
\end{align}
Consequently, uniform loss and update errors on $\mathcal B_R$ are
bounded by
\begin{equation}
 \varepsilon_{\mathcal L}=\tfrac{\varepsilon_H R^2}{2}+\tfrac{L_{\mathcal L}R^3}{6},
 \qquad
 \varepsilon_g=\varepsilon_HR+C_gR^2.
 \label{eq:endpoint-error-constants}
\end{equation}
\end{proposition}

\begin{proof}
The segment from $\theta_T$ to $\theta_T+e$ lies in the ball.
The root identities and the Hessian condition give
\[
 \left|\mathcal L(\theta_T+e)-\tfrac12e^\top H_Te\right|
 \le\tfrac{L_{\mathcal L}}{6}\norm e^3.
\]
Lemma~\ref{lem:score-update} gives the update remainder
$\norm{g(\theta_T+e)-H_Te}\le C_g\norm e^2$.
Replacing $H_T$ by $H_0$ adds at most $\varepsilon_H\norm e^2/2$ to the
loss error and $\varepsilon_H\norm e$ to the update error.
Using $\norm e\le R$ proves the uniform bounds.
\end{proof}

Lemma~\ref{lem:score-update} also yields
$\varepsilon_H\le L_H\norm\Delta$ with $L_H$ defined in
Eq~\eqref{eq:score-fisher-stability}.
For $R=c\norm\Delta$ with a fixed $c\ge1$, the preceding constants satisfy
\begin{align}
 \varepsilon_{\mathcal L}
 &\le\left(\tfrac{L_Hc^2}{2}+\tfrac{L_{\mathcal L}c^3}{6}\right)\norm\Delta^3,\\
 \varepsilon_g
 &\le\left(L_Hc+C_gc^2\right)\norm\Delta^2.
 \label{eq:small-displacement-errors}
\end{align}
These are uniform functional estimates on the stated ball.
The following lemma establishes how long the exact OPD flow remains
there for the choice $R=2\norm\Delta$.

\begin{lemma}[A local validity interval for the OPD mean flow]
\label{lem:opd-confinement}
Let $\rho=\norm{\theta_T-\theta_0}>0$ and suppose $g$ is
locally Lipschitz on an open neighborhood of the closed ball
$\mathcal B_{2\rho}$ about $\theta_T$.
Assume $g(\theta_T)=0$, $H_T\succeq0$, and
\[
 \norm{g(\theta_T+e)-H_Te}\le C_g\norm e^2
 \quad\text{on }\mathcal B_{2\rho}.
\]
The score conditions in Lemma~\ref{lem:score-update} supply this bound.
Define
\begin{equation}
 T_{\mathrm{loc}}:=
 \begin{cases}
  (2C_g\rho)^{-1},&C_g>0,\\
  +\infty,&C_g=0.
 \end{cases}
 \label{eq:opd-local-horizon}
\end{equation}
The exact OPD mean flow in Eq~\eqref{eq:opd-flow-alignment} exists uniquely and
remains in $\mathcal B_{2\rho}$ for $0\le t\le T_{\mathrm{loc}}$.
For $C_g>0$,
\begin{equation}
 \norm{\theta(t)-\theta_T}
 \le\frac{\rho}{1-C_g\rho t}\le2\rho,
 \qquad 0\le t\le T_{\mathrm{loc}}.
 \label{eq:opd-radius-bound}
\end{equation}
When $C_g=0$, the distance is nonincreasing. If $\rho=0$ and the field is
locally Lipschitz at its root, the solution is stationary at $\theta_T$.
\end{lemma}

\begin{proof}
Set $e=\theta-\theta_T$, $d=\norm e$, and
$q(e)=g(\theta_T+e)-H_Te$.
For the original flow $\dot e=-g(\theta_T+e)$, positivity
of $H_T$ and the remainder bound imply, whenever $d>0$,
\begin{equation}
 \dot d
 =-\frac{e^\top H_Te+\ip{e}{q(e)}}{d}
 \le C_gd^2.
 \label{eq:opd-radius-differential}
\end{equation}
If the root is reached, uniqueness makes the trajectory stationary there.
Otherwise, integrate the differential inequality for $1/d$ to obtain
the first bound in Eq~\eqref{eq:opd-radius-bound} up to any first exit.
For $t<T_{\mathrm{loc}}$, that bound is strictly below $2\rho$, ruling
out exit before $T_{\mathrm{loc}}$.
The solution stays in a compact ball contained in the open regularity
neighborhood, so the ODE continuation theorem gives existence through
$T_{\mathrm{loc}}$; continuity gives the closed-ball bound at its endpoint.
If $C_g=0$, Eq~\eqref{eq:opd-radius-differential} gives $d(t)\le\rho$,
and the same compact-continuation argument applies for every finite time.
The case $\rho=0$ follows from uniqueness at
$g(\theta_T)=0$.
\end{proof}

\section{Recovery bound from the parameter difference}
\label{app:main-proof}

Write $e(t)=\theta(t)-\theta_T$, where $\theta_T=\theta_{\mathrm{RL}}$,
and use the initial Fisher $H_0=H(\theta_0)$ from Eq~\eqref{eq:self-fisher}.
Define the reference energy and its low-curvature part by
\begin{equation}
 E_0(e):=\tfrac12e^\top H_0e,
 \qquad
 E_{\le\delta}(e):=E_0(P_{\le\delta}e).
 \label{eq:reference-energy}
\end{equation}
Because $e(0)=-\Delta$, Lemma~\ref{lem:endpoint-tail} gives
$E_{\le\delta}(e(0))=E_\Delta(\delta)$ exactly.
The theorem uses the error terms of Proposition~\ref{prop:endpoint-stability}
with $R=2\rho$,
\begin{equation}
 \begin{aligned}
 \varepsilon_{\mathcal L}(\rho)&=2\varepsilon_H\rho^2+\tfrac43L_{\mathcal L}\rho^3,\\
 \varepsilon_g(\rho)&=2\varepsilon_H\rho+4C_g\rho^2,
 \end{aligned}
 \label{eq:local-opd-errors}
\end{equation}
where $\varepsilon_H=\norm{H_T-H_0}_{\mathrm{op}}$ is defined in
Eq~\eqref{eq:teacher-fisher}.

\begin{theorem}[Local recovery for the nonlinear OPD mean flow]
\label{thm:detailed-recovery}
Consider exactly the flow in Eq~\eqref{eq:opd-flow-alignment} from
Proposition~\ref{thm:loss-dynamics}, starting at $\theta_0$, and let
$\theta_T=\theta_{\mathrm{RL}}$ with
$\rho=\norm{\theta_T-\theta_0}>0$. Assume:
\begin{enumerate}[label=(S\arabic*)]
 \item The teacher and student share their architecture, parameter
 coordinates, tokenizer, vocabulary, temperature, and policy interface.
 The quantities $\mathcal L$, $g$, and $H_0$ are well defined.
 The common-support, positivity and dominated-differentiation conditions
 of Lemma~\ref{lem:root-fisher} hold on a neighborhood of the closed ball
 $\mathcal B_{2\rho}$ about $\theta_T$. Thus
 $g(\theta_T)=0$ and
 $Dg(\theta_T)=H_T\succeq0$.

 \item The score and sampling conditions of
 Lemma~\ref{lem:score-update} hold on an open neighborhood of
 $\mathcal B_{2\rho}$ with constants $B,L_s\ge0$ and a finite total
 token-draw bound $M$. In particular, prompts have a fixed law,
 rollouts use the full student policy, stopping and masking rules do not
 depend on parameters, and $N(\tau)=|\cS(\tau)|>0$ almost surely.

 \item On an open neighborhood of $\mathcal B_{2\rho}$, the function
 $\mathcal L$ is $C^2$ with an $L_{\mathcal L}$-Lipschitz Hessian in operator norm.
\end{enumerate}
The root identities in (S1) follow from equality of student and teacher
policies; they require no stationarity of the teacher's RL objective.
Condition (S2) gives $C_g=\tfrac32B L_s+(M+1)B^3$ and local
Lipschitz continuity of the detached update.
Let $T_{\mathrm{loc}}$ be defined by Eq~\eqref{eq:opd-local-horizon}, and
let $\varepsilon_{\mathcal L}(\rho),\varepsilon_g(\rho)$ be the endpoint-dependent
quantities in Eq~\eqref{eq:local-opd-errors}.
The original OPD flow exists uniquely through $T_{\mathrm{loc}}$ and
remains in $\mathcal B_{2\rho}$.
For every finite $T_\star\le T_{\mathrm{loc}}$, every $\delta>0$, and
$0\le t\le T_\star$,
\begin{equation}
 \boxed{\begin{aligned}
 \mathcal L(\theta(t))\le{}&e^{-\delta t}\mathcal L(\theta_0)
                  +E_\Delta(\delta)\\
 &+2\varepsilon_{\mathcal L}(\rho)
  +\varepsilon_g(\rho)^2\left(\frac{t}{4}+\frac{1}{2\delta}\right).
 \end{aligned}}
 \label{eq:detailed-spectral-rate}
\end{equation}
The residual $E_\Delta(\delta)$ can be replaced by
$\delta\norm{\theta_{\mathrm{RL}}-\theta_0}^2/2$.
In the exact quadratic case with zero approximation errors,
\begin{equation}
 \mathcal L(\theta(t))
 \le e^{-2\delta t}\mathcal L(\theta_0)+E_\Delta(\delta).
 \label{eq:quadratic-spectral-rate}
\end{equation}
No lower bound on the smallest positive eigenvalue of either Fisher is
required. If $\rho=0$, local uniqueness at the root gives
$\theta(t)=\theta_T$ and $\mathcal L(\theta(t))=0$.

For a target tolerance $\varepsilon>0$, suppose some
$\delta>0$ and finite $T_\star\le T_{\mathrm{loc}}$ satisfy
\begin{equation}
 E_\Delta(\delta)+2\varepsilon_{\mathcal L}(\rho)
 +\varepsilon_g(\rho)^2\left(\frac{T_\star}{4}+\frac{1}{2\delta}\right)
 \le\frac\varepsilon2.
 \label{eq:spectral-tolerance-condition}
\end{equation}
Define
\begin{equation}
 t_\varepsilon
 :=\frac1\delta\log\max\left\{1,
       \frac{2\mathcal L(\theta_0)}{\varepsilon}\right\}.
 \label{eq:spectral-hitting-time}
\end{equation}
If $t_\varepsilon\le T_\star$, then $\mathcal L(\theta(t))\le\varepsilon$ for
every $t\in[t_\varepsilon,T_\star]$.
\end{theorem}

\begin{proof}
Lemma~\ref{lem:opd-confinement} gives existence and confinement of the
original OPD mean flow through $T_{\mathrm{loc}}$.
Applying Proposition~\ref{prop:endpoint-stability} with $R=2\rho$ gives
\begin{align}
 \left|\mathcal L(\theta(t))-E_0(e(t))\right|
 &\le\varepsilon_{\mathcal L}(\rho),
 \label{eq:loss-secant-stability}\\
 \norm{g(\theta(t))-H_0e(t)}
 &\le\varepsilon_g(\rho),
 \qquad 0\le t\le T_\star.
 \label{eq:update-secant-stability}
\end{align}
For brevity, write $\varepsilon_{\mathcal L}=\varepsilon_{\mathcal L}(\rho)$ and
$\varepsilon_g=\varepsilon_g(\rho)$ in the rest of the proof.
Let $E(t)=E_0(e(t))$ and
$r(t)=g(\theta(t))-H_0e(t)$.
The identity $\dot e=-H_0e-r$ rewrites the original update.
Young's inequality and Eq~\eqref{eq:update-secant-stability} imply
\begin{equation}
 \dot E(t)
 =-\norm{H_0e(t)}^2-\ip{H_0e(t)}{r(t)}
 \le-\tfrac12\norm{H_0e(t)}^2+\tfrac12\varepsilon_g^2.
 \label{eq:reference-energy-descent}
\end{equation}
Write $E_{\le\delta}(t)=E_{\le\delta}(e(t))$.
The spectral projector commutes with $H_0$. With
$a=\norm{H_0P_{\le\delta}e(t)}$,
\begin{align}
 \dot E_{\le\delta}(t)
 &=-a^2-\ip{H_0P_{\le\delta}e(t)}{P_{\le\delta}r(t)}\\
 &\le-a^2+\varepsilon_ga
 \le\frac{\varepsilon_g^2}{4}.
 \label{eq:low-energy-injection}
\end{align}
Using the endpoint value $E_{\le\delta}(0)=E_\Delta(\delta)$ gives
\begin{equation}
 E_{\le\delta}(t)
 \le E_\Delta(\delta)+\frac{\varepsilon_g^2t}{4}.
 \label{eq:low-energy-bound}
\end{equation}
The complementary eigenspace has eigenvalues greater than $\delta$, so
\begin{equation}
 \norm{H_0e(t)}^2
 \ge2\delta\bigl(E(t)-E_{\le\delta}(t)\bigr).
 \label{eq:high-spectrum-gap}
\end{equation}
Combining these inequalities yields
\begin{equation}
 \dot E(t)\le-\delta E(t)+\delta E_\Delta(\delta)
       +\frac{\delta\varepsilon_g^2t}{4}
       +\frac{\varepsilon_g^2}{2}.
 \label{eq:reference-energy-ode}
\end{equation}
Multiplying by $e^{\delta t}$ and integrating gives the convenient bound
\begin{equation}
 E(t)\le e^{-\delta t}E(0)+E_\Delta(\delta)
       +\varepsilon_g^2\left(\frac{t}{4}+\frac1{2\delta}\right).
 \label{eq:reference-spectral-rate}
\end{equation}
For example, the term linear in time follows by bounding $u\le t$ in
its convolution with $e^{-\delta(t-u)}$; the other forcing term contributes
at most $\varepsilon_g^2/(2\delta)$.
By Eq~\eqref{eq:loss-secant-stability},
$E(0)\le \mathcal L(\theta_0)+\varepsilon_{\mathcal L}$ and
$\mathcal L(\theta(t))\le E(t)+\varepsilon_{\mathcal L}$.
Since $e^{-\delta t}\le1$, substitution proves
Eq~\eqref{eq:detailed-spectral-rate}. Lemma~\ref{lem:endpoint-tail} gives the
parameter-distance bound.

In the exact quadratic limit, $\dot e=-H_0e$. Expanding in an eigenbasis,
with $c_i=u_i^\top e(0)=-u_i^\top\Delta$, gives
$\mathcal L(\theta(t))=\tfrac12\sum_i\lambda_ic_i^2e^{-2\lambda_it}$.
The terms with $\lambda_i>\delta$ are bounded by
$e^{-2\delta t}\mathcal L(\theta_0)$ and the remaining terms by
$E_\Delta(\delta)$, proving Eq~\eqref{eq:quadratic-spectral-rate}.
Finally, the error budget is bounded at $T_\star$ by
Eq~\eqref{eq:spectral-tolerance-condition}, and
$e^{-\delta t}\mathcal L(\theta_0)\le\varepsilon/2$ for
$t\ge t_\varepsilon$. The claimed interval is nonempty when
$t_\varepsilon\le T_\star$, completing the proof.
\end{proof}

\paragraph{Abridged form.}
Lemma~\ref{lem:score-update} gives $\varepsilon_H\le L_H\rho$.  Hence
$\varepsilon_{\mathcal L}(\rho)=O(\rho^3)$ and $\varepsilon_g(\rho)=O(\rho^2)$ in
Eq~\eqref{eq:local-opd-errors}.  For $t\le T_{\mathrm{loc}}=(2C_g\rho)^{-1}$, the
last two terms of Eq~\eqref{eq:detailed-spectral-rate} are therefore
$O(\rho^3+\rho^4/\delta)$, with constants depending only on $B$, $L_s$, $M$ and
$L_{\mathcal L}$.  Since $T_{\mathrm{loc}}=\Omega(1/\rho)$, this gives
Theorem~\ref{thm:main}.

\section{Local expansion and spectral specialization}
\label{app:local}

\begin{proposition}[Local same-model Fisher expansion]
\label{prop:local-expansion}
Assume the conditions of Proposition~\ref{prop:endpoint-stability}
on a neighborhood of $\theta_T$.
Let $H_T$ be the teacher Fisher in Eq~\eqref{eq:teacher-fisher}.
For $e=\theta-\theta_T$,
\begin{align}
 \mathcal L(\theta_T+e)
 &=\tfrac12e^\top H_Te+O(\norm e^3),
 \label{eq:local-F}\\
 g(\theta_T+e)
 &=H_Te+O(\norm e^2),
 \label{eq:local-g}\\
 h(\theta_T+e)
 &=O(\norm e^2),
 \label{eq:local-occ}\\
 \cA(\theta_T+e)
 &=e^\top H_T^2e+O(\norm e^3).
 \label{eq:local-A}
\end{align}
The full-logit forward KL has the same leading Hessian at the same-model
root.
\end{proposition}

\begin{proof}
Lemma~\ref{lem:root-fisher} gives $\nabla^2\mathcal L(\theta_T)=H_T$.
The Lipschitz Hessian condition gives Eq~\eqref{eq:local-F} and
$\nabla \mathcal L(\theta_T+e)=H_Te+O(\norm e^2)$.
Lemma~\ref{lem:score-update} gives Eq~\eqref{eq:local-g} directly from
the policy scores.
The identity $h=\nabla \mathcal L-g$ cancels their
linear terms, proving Eq~\eqref{eq:local-occ}.
Multiplying the expansions for $\nabla \mathcal L$ and $g$ gives
Eq~\eqref{eq:local-A}.
At equality of the two policies, the conditional forward and reverse KL
Hessians both equal the policy Fisher, yielding the final statement.
\end{proof}

This expansion is centered at the teacher and uses $H_T$.
Using $H_0$ as the reference requires controlling
$\norm{H_T-H_0}_{\mathrm{op}}$ in addition to the Taylor remainder.
Proposition~\ref{prop:endpoint-stability} gives that comparison explicitly
on the local ball. Lemma~\ref{lem:opd-confinement} then supplies a validity
interval for these estimates along the original nonlinear OPD flow.

\subsection{Fixed-Fisher specialization}

Suppose the local model is exactly quadratic with the initial Fisher:
\begin{equation}
 \mathcal L(\theta_T+e)=\tfrac12e^\top H_0e,
 \qquad
 g(\theta_T+e)=\nabla \mathcal L(\theta_T+e)=H_0e.
 \label{eq:fixed-fisher}
\end{equation}
Then $\dot e=-H_0e$ and $e(0)=-\Delta$.
Let $H_0u_i=\lambda_i u_i$ and $c_i=-u_i^\top\Delta$.
The exact solution and loss are
\begin{align}
 e(t)&=\sum_i c_i e^{-\lambda_it}u_i,\\
 \mathcal L(\theta(t))&=\tfrac12\sum_i\lambda_ic_i^2e^{-2\lambda_it},
 \label{eq:spectral-F}\\
 \gamma(t)&=
 \frac{\sum_i\lambda_i^2c_i^2e^{-2\lambda_it}}
      {\sum_i\lambda_ic_i^2e^{-2\lambda_it}},
 \qquad \mathcal L(\theta(t))>0.
 \label{eq:spectral-mu}
\end{align}
For normalized weights
$\omega_i(t)\propto\lambda_ic_i^2e^{-2\lambda_it}$,
\begin{equation}
 \dot\gamma(t)=-2\operatorname{Var}_{\omega(t)}(\lambda)\le0.
 \label{eq:mu-decay}
\end{equation}
Higher-curvature components decay first, so this decline need not represent
a failure to fit the teacher if the remaining loss is already small.

For any cutoff $\delta>0$, splitting the spectral sum yields
\begin{equation}
 \mathcal L(\theta(t))
 \le E_\Delta(\delta)+e^{-2\delta t}\mathcal L(\theta_0).
 \label{eq:tolerance-rate}
\end{equation}
If $E_\Delta(\delta)\le\varepsilon/2$, then loss at most $\varepsilon$
is reached once
$t\ge(2\delta)^{-1}\log\max\{1,2\mathcal L(\theta_0)/\varepsilon\}$,
provided the quadratic model remains valid until that time.
The residual is computed directly from the initial parameter difference
and also satisfies $E_\Delta(\delta)\le\delta\norm\Delta^2/2$.

\paragraph{Equal distances can have different relative decay rates.}
Take $H_0=\operatorname{diag}(1,\eta)$ with $0<\eta<1$, and compare
parameter differences $\Delta_1=(d,0)^\top$ and
$\Delta_2=(0,d)^\top$ for $d>0$.
Both have norm $d$, but Eq~\eqref{eq:spectral-F} gives
\begin{equation}
 \frac{F^{(1)}(t)}{F^{(1)}(0)}=e^{-2t},
 \qquad
 \frac{F^{(2)}(t)}{F^{(2)}(0)}=e^{-2\eta t}.
 \label{eq:equal-distance-example}
\end{equation}
Thus even in the exact quadratic case, an arbitrarily small parameter
distance does not specify a uniform relative convergence rate.
The direction of the parameter difference matters alongside its magnitude.

\clearpage
\section{Experimental details for the teacher comparisons}
\label{app:teacher-comparison-details}

\setcounter{table}{0}
\renewcommand{\thetable}{\thesection.\arabic{table}}
\renewcommand{\theHtable}{\thesection.\arabic{table}}

\paragraph{Figure layout.}
Each column of Figure~\ref{fig:teacher-trainability} fixes the student named
in its heading and identifies the teachers in its color key. The top row
shows macro-averaged validation accuracy; the bottom row shows the recorded
OPD training loss. Solid curves show the student runs, and colors identify
the corresponding teachers. The reference lines, loss statistic, and
comparison horizons are detailed below.

\paragraph{Models and training data.}
Table~\ref{tab:teacher-setup} lists the three student/task groups, the seven
teacher conditions, and the OPD training datasets used in
Figure~\ref{fig:teacher-trainability}.
For code, we distill \QwenRLCode{}~\citep{yang2026learning} and
Qwen3-14B~\citep{yang2025qwen3} into Qwen3-4B.  For math, we distill
\QwenRLMath{}~\citep{yang2026learning} and Qwen3-4B~\citep{yang2025qwen3} into
Qwen3-1.7B, and JustRL-1.5B~\citep{he2025justrl},
\SkyworkTeacher{}~\citep{he2025skywork}, and
R1-Distill-7B~\citep{guo2025deepseek} into R1-Distill-1.5B.  All Qwen3 models
use non-thinking mode, and R1-Distill denotes DeepSeek-R1-Distill-Qwen.
\QwenRLCode{} and \QwenRLMath{} are obtained by RL post-training
\QwenFourNT{} on code and math tasks, respectively~\citep{yang2026learning}.
JustRL-1.5B and \SkyworkTeacher{} are obtained by RL post-training
R1-Distill-1.5B and R1-Distill-7B, respectively, on math
tasks~\citep{he2025justrl,he2025skywork}.
For \QwenRLCode{} and \QwenRLMath{}, we use the released step-300
and step-500 checkpoints, respectively.
The table separately identifies the teacher's RL training data in the note
and the student's OPD prompts in the last column.
In particular, \QwenRLMath{} is trained on filtered DeepMath, while both math
OPD groups use DAPO-Math-17K.

The short names used in the text expand as follows:
\QwenRLCode{} and \QwenRLMath{} denote Qwen3-4B-\allowbreak Non-Thinking-RL-Code and
Qwen3-4B-\allowbreak Non-Thinking-RL-Math; R1-Distill-1.5B and R1-Distill-7B denote
DeepSeek-R1-Distill-Qwen-1.5B and DeepSeek-R1-Distill-Qwen-7B;
JustRL-1.5B denotes JustRL-DeepSeek-1.5B; and \SkyworkTeacher{} denotes
Skywork-OR1-Math-7B. \QwenRLCode{} and JustRL-1.5B are self-RL teachers for
their corresponding students, whereas \QwenRLMath{} and \SkyworkTeacher{} are
RL-trained models larger than their students.

\begin{table}[htbp]
 \centering
 \normalsize
 \setlength{\tabcolsep}{3pt}
 \renewcommand{\arraystretch}{1.15}
 \caption{Student--teacher groups and OPD training data. Qwen3
 models~\citep{yang2025qwen3} use non-thinking mode; R1-Distill abbreviates
 DeepSeek-R1-Distill-Qwen~\citep{guo2025deepseek}.}
 \label{tab:teacher-setup}
 \begin{tabular}{@{}p{0.20\linewidth}p{0.56\linewidth}p{0.21\linewidth}@{}}
  \hline
  \textbf{Task / student} & \textbf{Teachers} & \textbf{OPD data} \\
  \hline
  \raggedright Code /\newline Qwen3-4B\newline (Non-thinking)
   & \QwenFourteenNT{}\newline \QwenRLCode{}~\citep{yang2026learning}
   & \raggedright Eurus-2-RL-Data (code)~\citep{cui2025process} \tabularnewline
  \hline
  \raggedright Math /\newline Qwen3-1.7B\newline (Non-thinking)
   & \QwenFourNT{}\newline \QwenRLMath{}~\citep{yang2026learning}
   & \raggedright DAPO-Math-17K~\citep{yu2026dapo} \tabularnewline
  \hline
  Math /\newline R1-Distill-1.5B
   & R1-Distill-7B\newline
     JustRL-DeepSeek-1.5B~\citep{he2025justrl}\newline
     Skywork-OR1-Math-7B~\citep{he2025skywork}
   & \raggedright DAPO-Math-17K \tabularnewline
  \hline
 \end{tabular}
 \par\vspace{3pt}
 \begin{minipage}{\linewidth}
  \normalsize
  \textit{RL provenance.} \QwenRLCode{} and \QwenRLMath{} are GRPO checkpoints of
  Qwen3-4B trained on ${\sim}25$K Eurus coding problems and 57K
  DeepMath problems~\citep{he2026deepmath} (difficulty $\geq6$), respectively.
  JustRL-1.5B and \SkyworkTeacher{} are RL-post-trained from R1-Distill-1.5B and
  R1-Distill-7B, respectively.
 \end{minipage}
\end{table}

\paragraph{Training configurations.}
Table~\ref{tab:opd-training-config} lists the settings for all seven main
student--teacher runs, verified against their saved training configurations.
The two Qwen3-1.7B runs share the same DAPO-Math-17K prompts and validation
protocol; validation includes the initial checkpoint and every ten updates.

\begin{table}[htbp]
  \centering
  \normalsize
  \setlength{\tabcolsep}{4pt}
  \renewcommand{\arraystretch}{1.08}
  \caption{Training configurations for the seven main OPD runs. Each teacher
  condition is one training run; the three columns contain two, two, and
  three teacher conditions, respectively.}
  \label{tab:opd-training-config}
  \begin{tabular}{@{}lccc@{}}
    \hline
    \textbf{Setting} & \shortstack{\textbf{Code}\\Qwen3-4B\\(Non-thinking)}
      & \shortstack{\textbf{Math}\\Qwen3-1.7B\\(Non-thinking)}
      & \shortstack{\textbf{Math}\\R1-Distill-1.5B\\\strut} \\
    \hline
    OPD objective & \multicolumn{3}{c}{Reverse KL, Eq~\eqref{eq:online-loss}} \\
    Vocabulary approximation & \multicolumn{3}{c}{Student's top 16 tokens} \\
    Token weights & \multicolumn{3}{c}{Student probabilities} \\
    Loss normalization & \multicolumn{3}{c}{Global mean over valid response tokens} \\
    \hline
    Optimizer & AdamW & AdamW & AdamW \\
    Learning rate & $10^{-6}$ & $10^{-6}$ & $10^{-6}$ \\
    Learning-rate schedule & constant & constant & constant \\
    Warmup fraction & 0 & 0 & 0 \\
    Adam $(\beta_1,\beta_2)$ & $(0.9, 0.999)$ & $(0.9, 0.999)$ & $(0.9, 0.999)$ \\
    Weight decay & 0.01 & 0.01 & 0.01 \\
    Gradient clipping ($\ell_2$ norm) & 1 & 1 & 1 \\
    Prompt batch size & 64 & 64 & 64 \\
    Optimization minibatch size & 64 & 64 & 64 \\
    Training rollouts per prompt & 1 & 1 & 1 \\
    Optimization epochs per batch & 1 & 1 & 1 \\
    Maximum prompt tokens & 2,048 & 1,024 & 1,024 \\
    Maximum response tokens & 16,384 & 7,168 & 21,000 \\
    Student sampling temperature & 1 & 1 & 1 \\
    Teacher temperature & 1 & 1 & 1 \\
    Student sampling top-$p$ & 1 & 1 & 1 \\
    Configured training updates & 200 & 200 & 200 \\
    Validation interval (updates) & 10 & 10 & 10 \\
    Checkpoint interval (updates) & 10 & 10 & 10 \\
    \hline
    Common comparison endpoint & 200 & 200 & 200 \\
    \hline
  \end{tabular}
  \par\vspace{3pt}
  \begin{minipage}{\linewidth}
    \normalsize
    Student rollouts provide the prefixes; the teacher is fixed during training.
    Merged entries apply to all three groups. The configured budget is 200 updates
    in every run, and all seven comparisons use the complete 200-update window.
    The Code run with \QwenRLCode{} is the completed reproduction described in
    the comparison-scope discussion below.
  \end{minipage}
\end{table}

\paragraph{Validation protocol.}
Code validation covers HumanEval+ and MBPP+~\citep{liu2023your} and
LiveCodeBench~\citep{jain2025livecodebench}; math validation covers
AIME 2024, AIME 2025, and AMC 2023.
We report the unweighted mean of the three benchmark scores, using
\(\mathrm{avg}@4\) for code and \(\mathrm{avg}@8\) for math.
These quantities average sampled outcomes; they are not the probability
that at least one of $k$ samples is correct, commonly denoted
\texttt{pass@k}.
The Code and R1-Distill standalone teacher references use the benchmarks,
decoding parameters, and response-length limits of their corresponding
student evaluations. For the completed \QwenRLCode{} run, the saved
validation outputs contain eight responses per problem. We select four
without replacement using a fixed seed of 0 and recompute
\(\mathrm{avg}@4\). Selection depends only on the benchmark, prompt identity,
and response index, not on correctness; each problem uses the same four
response indices at every checkpoint.
Color-matched dashed lines in Figure~\ref{fig:teacher-trainability} show
standalone teacher accuracy, with each reference line labeled directly.
They provide reference levels, not upper bounds on student performance.
For the Qwen3-1.7B math runs, validation uses 30 AIME 2024 problems,
30 AIME 2025 problems, and 83 AMC 2023 problems, with eight responses per
problem, temperature $0.7$, top-$p=0.95$, and a response limit of 31,744
tokens. Each checkpoint therefore evaluates 1,144 responses over 143
distinct problems. The two Qwen-math teachers are evaluated under this same
protocol, giving macro accuracies of $69.17\%$ for \QwenRLMath{} (AIME 2024
$63.75$, AIME 2025 $56.25$, AMC 2023 $87.50$) and $35.13\%$ for Qwen3-4B
Non-thinking ($25.42$, $19.58$, $60.39$); the earlier archived scores under
a 7,168-token limit are not used.

\paragraph{Training-loss statistic.}
The plotted loss is the global token-mean student-top-16 OPD surrogate:
the negative mean of the top-$K$ advantage summed over $K$ and averaged
over valid response tokens.
We use the logged statistic
\path{critic/advantages/mean_sum_over_k} with its sign reversed,
which removes the packed-microbatch scaling present in
\texttt{actor/pg\_loss}.
This implementation-level surrogate is distinct from the full-vocabulary
KL objective analyzed in Section~\ref{sec:main}.
The training signal $\widehat\mu_m$ of
Section~\ref{sec:learning-signal-collapse} uses the same $\widehat{\mathcal L}_m$
together with the logged pre-clip gradient norm \path{actor/grad_norm}.
In Table~\ref{tab:learning-signal}, the signal before the split step $s$ is
the median of $\widehat\mu_m$ over steps 1 to $s$, the signal after $s$ is the
median over steps $s+1$ to $s+30$, the last-50-steps column is the median
over the run's last 50 steps, and the loss left at $s$ is the centered
9-step rolling median of $\widehat{\mathcal L}_m/\widehat{\mathcal L}_1$ (the
smoothed curve of Figure~\ref{fig:signal-collapse-timeseries}) at step $s$.  The
signal gaps quoted in Section~\ref{sec:learning-signal-collapse} are ratios
of the after-$s$ medians of the runs with self-RL and larger-scale
teachers.
Figure~\ref{fig:teacher-trainability} draws the recorded $\widehat{\mathcal L}_m$
faintly and overlays its centered 9-step rolling median, without
normalization; all loss statistics use the raw records. Every loss panel displays its own numerical tick labels;
vertical limits vary by student/task group. Fractional loss reduction is
reported separately in Table~\ref{tab:loss-ability-recovery}.

\paragraph{Learning signal over training steps.}
Figure~\ref{fig:signal-collapse-timeseries} shows the data of
Figure~\ref{fig:code-signal-collapse} against the training step, and
Table~\ref{tab:learning-signal} summarizes the learning signal before and
after a split step $s$.

\begin{figure}[tbp]
 \centering
 \includegraphics[width=\textwidth]{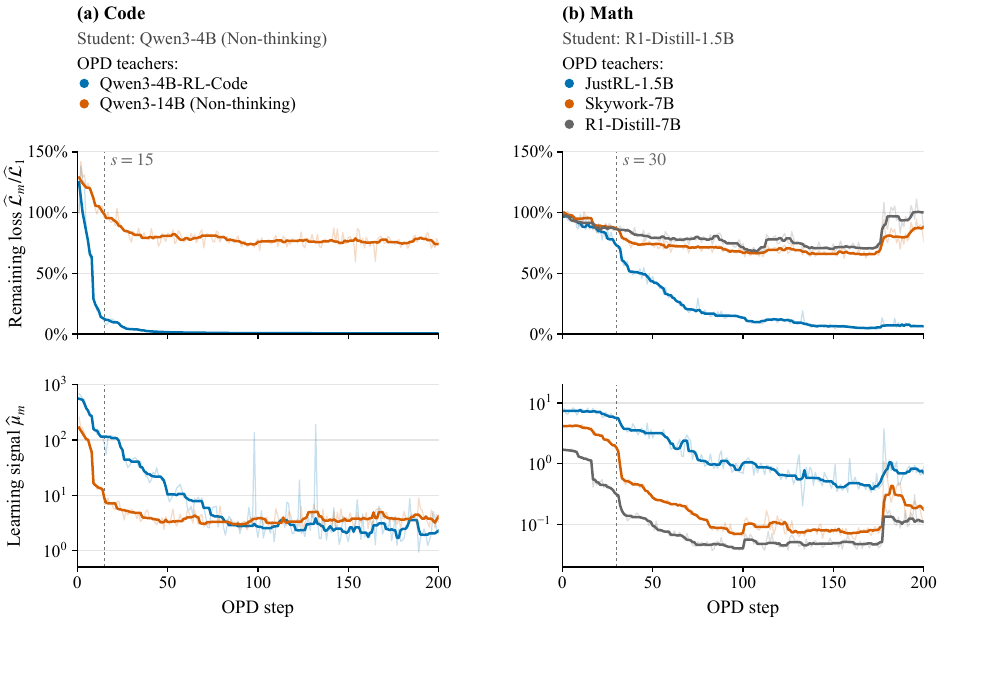}
 \caption{Learning signal and remaining loss over training steps.  Left
 column: the Code comparison; right column: the Math comparison.  Top:
 remaining OPD loss $\widehat{\mathcal L}_m/\widehat{\mathcal L}_1$.  Bottom:
 learning signal $\widehat\mu_m$ on a log scale. Both settings cover all
 200 updates. Lines are centered 9-step
 rolling medians of the per-step values, which are drawn faintly behind them.
 The dotted vertical line marks the split step $s$ of
 Table~\ref{tab:learning-signal}.}
 \label{fig:signal-collapse-timeseries}
\end{figure}

\begin{table}[htbp]
  \centering
  \normalsize
  \setlength{\tabcolsep}{5pt}
  \renewcommand{\arraystretch}{1.08}
  \caption{Learning signal $\widehat\mu_m$ around a split step $s$.  For each run of Figure~\ref{fig:signal-collapse-timeseries} we report
  the median learning signal $\widehat\mu_m$ before the split step $s$ (steps
  1 to $s$), in the 30 steps after $s$, and in the last 50 steps of the run;
  the drop is the ratio of the first two; loss left is the smoothed remaining
  loss $\widehat{\mathcal L}_m/\widehat{\mathcal L}_1$ at step $s$.  Here $s=15$
  for Code and $s=30$ for Math.  $^{\dagger}$ marks self-RL teachers.}
  \label{tab:learning-signal}
  \begin{tabular}{@{}lccccc@{}}
    \hline
    Teacher & Before $s$ & After $s$ & Drop & Loss left at $s$ & Last 50 steps \\
    \hline
    \multicolumn{6}{@{}l}{\emph{Code, $s=15$}} \\
    \QwenRLCode{}$^{\dagger}$ & $272$ & $40.9$ & $6.7\times$ & $12.6\%$ & $2.37$ \\
    Qwen3-14B & $60.7$ & $5.05$ & $12.0\times$ & $98.8\%$ & $3.78$ \\
    \hline
    \multicolumn{6}{@{}l}{\emph{Math, $s=30$}} \\
    JustRL-1.5B$^{\dagger}$ & $7.11$ & $3.37$ & $2$$\times$ & $73.0\%$ & $0.56$ \\
    \SkyworkTeacher{} & $3.73$ & $0.38$ & $10$$\times$ & $86.5\%$ & $0.090$ \\
    R1-Distill-7B & $1.14$ & $0.121$ & $9$$\times$ & $86.5\%$ & $0.056$ \\
    \hline
  \end{tabular}
\end{table}

\paragraph{Cross-fitted estimator of $\alpha$.}
At each checkpoint $m$ of Section~\ref{sec:learning-signal-collapse}, the
eight rollouts per prompt are split into two folds.  Let $g_m^{(0)}$ and
$g_m^{(1)}$ be the detached top-16 loss gradients estimated from the two
folds, and $h_m^{(0)}$ and $h_m^{(1)}$ the likelihood-ratio estimates of the
occupancy gradient, each using the opposite fold to construct its per-prompt
baseline.  The estimators are
\begin{equation}
 \widehat g_m^2:=\ip{g_m^{(0)}}{g_m^{(1)}},
 \qquad
 \widehat X_m:=\tfrac12\bigl(\ip{h_m^{(0)}}{g_m^{(1)}}+\ip{h_m^{(1)}}{g_m^{(0)}}\bigr),
 \qquad
 \widehat\alpha_m:=\widehat X_m/\widehat g_m^2.
 \label{eq:crossfit-occupancy}
\end{equation}
Conditional on the sampled prompts, the rollout folds are independent.
Their cross-product removes the same-fold rollout-noise bias from
$\widehat g_m^2$, but not uncertainty from prompt sampling. The baseline
from the opposite fold does not bias $\widehat X_m$.  Code uses 16
prompts and Math 64 prompts per checkpoint; Table~\ref{tab:occupancy-factor}
uses all 20 diagnostic checkpoints at steps $10,20,\ldots,200$ in both
settings. The Code learning-signal curves use all 200 updates of the same
runs as Figure~\ref{fig:teacher-trainability}: the completed \QwenRLCode{}
reproduction and the original Qwen3-14B run. The Qwen3-14B occupancy
statistics use all 20 checkpoints of that original run.

\paragraph{Metric reporting.}
Accuracy gain and teacher--student gap recovery follow the definitions in
Section~\ref{sec:evaluation-metrics}. Loss minima in
Table~\ref{tab:loss-ability-recovery} use raw updates 1--200 and need not be
sustained. For the five larger-scale runs and two self-RL runs, mean maximum
reductions are $41.6\%$ and $97.9\%$, respectively; mean final reductions
$(1-\widehat{\mathcal L}_{200}/\widehat{\mathcal L}_{1})\times100\%$
are $25.1\%$ and $96.2\%$.

\paragraph{Comparison scope.}
All seven main comparisons in Figure~\ref{fig:teacher-trainability} and
Table~\ref{tab:loss-ability-recovery} use 200 training updates. For
Code/Qwen3-4B, the \QwenRLCode{} curve comes from an independent reproduction
on two H200 GPUs, with the same OPD hyperparameters and native checkpoint
resumes. It replaces the earlier run that stopped at update 149; the two
runs are not concatenated. Qwen3-14B uses its complete original 200-update
run. Code accuracy uses four responses per problem, including the fixed-seed
subsampling described above. Validation is recorded at steps
$0,10,\ldots,200$ for the \QwenRLCode{} reproduction and the two
Math/Qwen3-1.7B runs. The Qwen3-14B and three Math/R1-Distill-1.5B runs
originally begin validation at step~10. We supplement their curves with
evaluations of the same initial student under the corresponding validation
protocol: Code shares the step~0 evaluation from the \QwenRLCode{}
reproduction, and R1-Distill-1.5B shares a before-training evaluation from
a separate full-logit OPD run with JustRL-1.5B. These are shared measurements
of the initial model, not additional evaluations within each training run.
The two Math/Qwen3-1.7B curves retain their own before-training evaluations
of the same base checkpoint; their step~0 estimates differ slightly because
of sampling. In all seven comparisons, accuracy gains and gap recovery use
the base student model at step~0 as $\mathrm{Acc}_{\mathrm{start}}$.
The loss and validation panels within each group both extend to step~200.
Each panel uses a linear scale that contains all its plotted observations.
Validation accuracy and loss use column-specific vertical limits, so their
absolute magnitudes are compared within the corresponding student/task group.
Each teacher condition comprises one training run; the curves therefore do
not provide repeated-seed uncertainty estimates.

\begin{table}[htbp]
\centering
\normalsize
\setlength{\tabcolsep}{5pt}
\caption{Completed 200-update \QwenOneSevenNT{} math runs. $n$ is the number of training rollouts per prompt; validation always uses eight responses per problem. Accuracies are percentages and gain is in percentage points. Final loss reduction is $(1-\widehat{\mathcal L}_{200}/\widehat{\mathcal L}_1)\times100\%$, using each run's own objective. Each row is one training run.}
  \label{tab:qwen-math-200-ablation}
  \begin{tabular}{llrrrrr}
\hline
Teacher & Support & $n$ & $\mathrm{Acc}_0$ & $\mathrm{Acc}_{200}$ & Gain & \shortstack{Final loss\\reduction} \\
\hline
\QwenRLMath{} & Top-16 & 1 & 21.10 & 43.94 & 22.83 & 55.37\% \\
\QwenFourNT{} & Top-16 & 1 & 20.66 & 22.74 & 2.07 & 6.24\% \\
\QwenRLMath{} & Full logits & 1 & 20.55 & 46.36 & 25.81 & 54.85\% \\
\QwenFourNT{} & Full logits & 1 & 21.66 & 22.88 & 1.22 & 11.60\% \\
\QwenFourNT{} & Top-16 & 8 & 21.85 & 24.87 & 3.02 & 9.89\% \\
\hline
\end{tabular}

\end{table}

\paragraph{Additional completed Qwen-math comparisons.}
Table~\ref{tab:qwen-math-200-ablation} includes the two main top-16 runs,
two full-logit runs with one training rollout per prompt, and a top-16
run with eight training rollouts per prompt for the Qwen3-4B teacher.
All rows use the Qwen-math training and validation limits described above.
The full-logit runs replace top-16 support with all 151,936 student
vocabulary entries. Each loss reduction is computed using that run's own
objective at updates 1 and 200; top-16 and full-logit loss magnitudes are
not identical objectives.
With full logits, \QwenRLMath{} yields a $25.81$-point accuracy gain and a
$54.85\%$ loss reduction, while Qwen3-4B yields a $1.22$-point gain and an
$11.60\%$ loss reduction. The teacher ordering observed with top-16 OPD
therefore also appears with full logits in these runs.
For the Qwen3-4B teacher, the eight-rollout run gains $3.02$ accuracy
points and reduces its loss by $9.89\%$. This additional single run does
not establish a general effect of rollout count across teachers.
The additional runs are excluded from the seven-run averages in
Section~\ref{sec:phenomenon-results}.

\paragraph{Scope of the additional experiments.}
The seven main OPD runs are supplemented by eight Math runs that vary
vocabulary support or the number of training rollouts
(Appendix~\ref{app:support-rollout-ablations}). Three of these eight runs
already appear in Table~\ref{tab:qwen-math-200-ablation}; they are not
additional replications. Appendix~\ref{app:ablation-sft} reports four SFT
runs and two subsequent OPD runs initialized from SFT checkpoints.
Appendix~\ref{app:if-experiment} reports two instruction-following OPD
runs. These comparisons use the specified single-run configurations;
they do not constitute repeated-seed estimates. Checkpoint diagnostics
and standalone model evaluations are reported separately from training
runs.

\subsection{Additional Training-Dynamics Diagnostics}
\label{app:additional-dynamics}

\paragraph{Gradient norms and clipping.}
Figure~\ref{fig:appendix-gradient-norm} separates the gradient norm from the
loss-normalized learning signal in Figure~\ref{fig:signal-collapse-timeseries}.
We use the logged pre-clipping norm at each of the 200 optimizer updates.
The norm exceeds the clipping threshold of 1 at 24 updates with
\QwenRLCode{}, 64 with Qwen3-14B, 56 with JustRL-1.5B, 22 with
\SkyworkTeacher{}, and none with R1-Distill-7B.
These counts describe when clipping is active; they do not identify the
effective AdamW update, which also depends on the optimizer's moment
estimates and weight decay. Accordingly, the ratio
$\widehat\mu_m=\widehat G_m^2/(2\widehat{\mathcal L}_m)$ is an empirical
diagnostic rather than a direct measurement of the population
continuous-time loss derivative. A small gradient norm alone does not
distinguish successful matching from a plateau; the remaining loss must
also be considered. All recorded losses in these five diagnostic runs are
positive. The truncated objective is not a nonnegative full-vocabulary KL
in general; Appendix~\ref{app:if-experiment} reports a run in which it
becomes negative and does not interpret the corresponding ratio as a
loss-decay rate.

\begin{figure}[tbp]
  \centering
  \includegraphics[width=\textwidth]{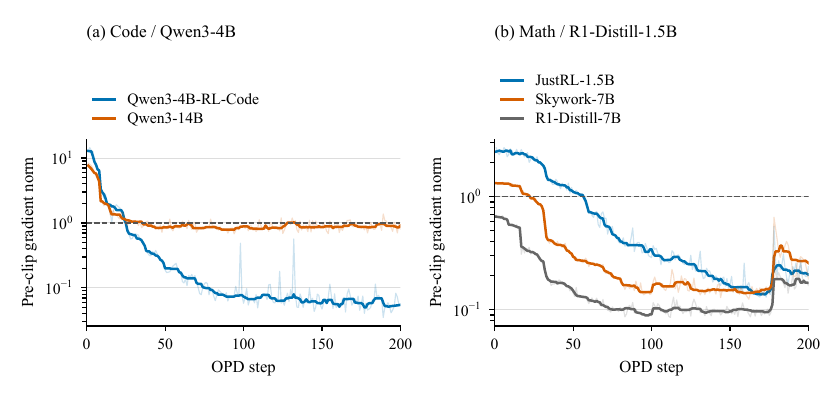}
  \caption{Pre-clipping gradient norms in the five runs used for the
  learning-signal analysis. Curves use a logarithmic vertical axis.
  Solid lines are centered 9-step rolling medians; faint lines show all
  200 recorded values. The dashed line marks the clipping threshold of 1.
  These are norms of the training gradients, not norms of the AdamW
  parameter updates.}
  \label{fig:appendix-gradient-norm}
\end{figure}

\paragraph{Diagnostic sampling.}
The occupancy measurements are separate evaluations of saved checkpoints,
not additional training runs. The prompt subset is fixed across the 20
checkpoints at steps $10,20,\ldots,200$, with eight sampled responses per
prompt at each checkpoint. The Code subset contains 16 Eurus coding
prompts; the Math subset contains the first
64 DAPO-Math-17K prompts after deterministic overlength filtering, without
shuffling. The Math diagnostic uses sampling seed 42 and FP32 model
computations. Table~\ref{tab:occupancy-diagnostic-config} summarizes the
sampling budget. Four responses per prompt enter each fold of
Eq~\eqref{eq:crossfit-occupancy}; the opposite fold supplies the per-prompt
baseline for the occupancy gradient.

\begin{table}[htbp]
  \centering
  \caption{Sampling settings for checkpoint-level occupancy diagnostics.
  These rollout counts differ from the one-rollout training setting.}
  \label{tab:occupancy-diagnostic-config}
  \setlength{\tabcolsep}{5pt}
  \renewcommand{\arraystretch}{1.12}
  \begin{tabular}{@{}lcc@{}}
    \toprule
    Setting & Code & Math \\
    \midrule
    Student & Qwen3-4B & R1-Distill-1.5B \\
    Prompt source & Eurus code & DAPO-Math-17K \\
    Prompts per checkpoint & 16 & 64 \\
    Responses per prompt & 8 & 8 \\
    Responses per fold & 64 & 256 \\
    Maximum response length & 16,384 & 21,000 \\
    Student vocabulary support & Top 16 & Top 16 \\
    Checkpoints per teacher & 20 & 20 \\
    \bottomrule
  \end{tabular}
\end{table}

\paragraph{Occupancy factor over training.}
Figure~\ref{fig:appendix-occupancy-dynamics} shows every measured
$1+\widehat\alpha_m$, including the self-RL teacher JustRL-1.5B as a
reference. All 80 cross-fold estimates of $\widehat g_m^2$ are positive,
so no checkpoint is removed on account of a nonpositive denominator.
For the larger-scale teachers, $1+\widehat\alpha_m$ is positive at
19/20 checkpoints with Qwen3-14B, 12/20 with \SkyworkTeacher{}, and 13/20
with R1-Distill-7B, reproducing Table~\ref{tab:occupancy-factor}.
It is also positive at 13/20 checkpoints with JustRL-1.5B.
The checkpoint estimates fluctuate and can become negative even in the
self-RL run. They therefore do not establish the sign of the population
loss derivative at every update. In particular, these finite-sample ratios
should not be read as confidence intervals or as deterministic predictions
of the subsequent AdamW trajectory. Their role is to test whether
persistent occupancy cancellation is shared by the observed plateaus.

\begin{figure}[tbp]
  \centering
  \includegraphics[width=\textwidth]{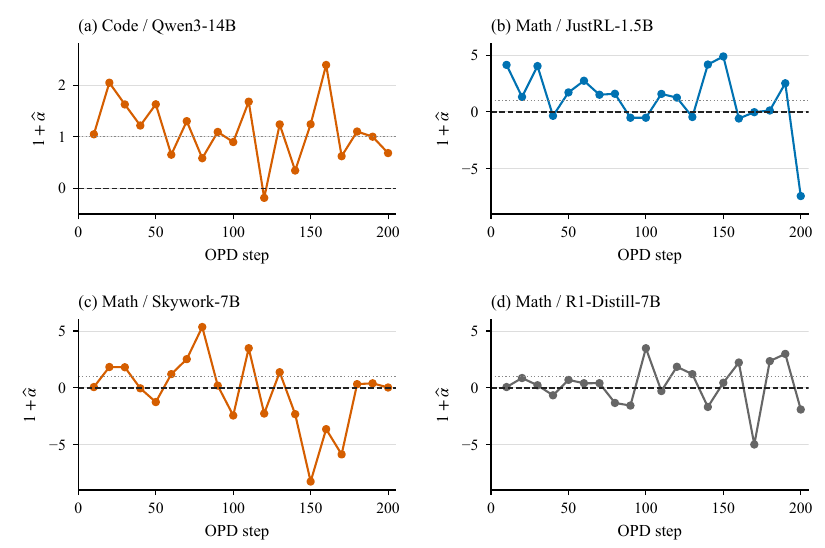}
  \caption{Checkpoint-level occupancy factor $1+\widehat\alpha_m$.
  Each point is one cross-fitted estimate, and lines connect consecutive
  checkpoints without smoothing. The dashed line is zero, the cancellation
  threshold of the idealized flow; the dotted line is one, corresponding
  to zero occupancy contribution. The Code panel uses a different vertical
  range from the three Math panels. No confidence intervals are implied.}
  \label{fig:appendix-occupancy-dynamics}
\end{figure}

\clearpage
\section{Parameter and Representation Measurement Details}
\label{app:representation-details}
\setcounter{table}{0}

\paragraph{Checkpoint scope and parameter statistics.}
The eight OPD endpoints in Table~\ref{tab:parameter-representation} are
full-parameter, student-top-16 reverse-OPD checkpoints at step 200.
The three R1-Distill-1.5B runs use the same clean single-rollout runs as
Figure~\ref{fig:teacher-trainability}: JustRL-1.5B, \SkyworkTeacher{}, and R1-Distill-7B.
The Qwen3-4B endpoints use Qwen3-14B Non-thinking and \QwenRLCode{} on
Code. The \QwenRLCode{}
endpoint comes from a 200-step reproduction on two H200 GPUs; the original
run stopped at update 149. The two Qwen3-1.7B endpoints come from
the Math runs of Figure~\ref{fig:teacher-trainability}, with teachers
\QwenRLMath{} and Qwen3-4B. The IF endpoint comes from the supplementary run of
Appendix~\ref{app:if-experiment} with the self-RL teacher
UltraData-IF-1.5B.
All parameter statistics compare exported model checkpoints, compute
tensor differences in FP32, and count tied embedding/head weights once.
The global displacement aggregates the squared differences and squared
starting weights before taking their ratio and square root; it is not an
unweighted average of per-layer relative changes.

\paragraph{Teacher--student parameter distance.}
The Teacher columns of Table~\ref{tab:parameter-representation} report the relative
parameter distance $\norm{\theta_T-\theta_0}_2/\norm{\theta_0}_2$ between each
self-RL teacher and its student, which
Section~\ref{sec:parameter-recovery} uses.  We compute it with the same
protocol as $d_{\mathrm{par}}$ above, comparing the released teacher checkpoint
with the student checkpoint used in our runs.  Most parameter entries are
identical in BF16, namely 87.7\% for \QwenRLCode{}, 34.4\% for JustRL-1.5B and
83.4\% for UltraData-IF-1.5B.  In the JustRL-1.5B run, the OPD student moves by
only 0.098\%, much less than the teacher's distance of 0.50\%.  OPD thus
reduces the loss to 3.7\% of its starting value without moving the student to
the teacher's parameters.

\paragraph{Representation extraction and CKA.}
Each checkpoint pair uses identical tokenizer vocabulary and prompt token
IDs. The Math bank contains AIME24 (30 prompts), AIME25 (30), and AMC23
(83); the Code bank contains HumanEval+ (164), MBPP+ (378), and
LiveCodeBench (131). A BF16 forward pass is performed separately for each
prompt, without response generation. At each hidden-state index, token
states are mean-pooled and stored in FP32. For the resulting representation
matrices $X$ and $Y$, centered across prompts, we compute
\begin{equation}
 \operatorname{CKA}(X,Y)
 =\frac{\norm{Y^\top X}_F^2}
        {\norm{X^\top X}_F\norm{Y^\top Y}_F}
 \label{eq:prompt-cka}
\end{equation}
using FP64 accumulation~\citep{kornblith2019similarity}.
The Teacher columns of Table~\ref{tab:parameter-representation} use the same protocol
and the prompt bank of the corresponding base model, so UltraData-IF-1.5B is
also compared on the Math bank.  The IF endpoint is likewise compared on
the R1-Distill-1.5B Math bank.  The Qwen3-1.7B endpoints use the Qwen3 Math
bank, since Qwen3-1.7B and Qwen3-4B share the same tokenizer.  The node that
holds the \QwenRLMath{} endpoint had no free GPU, so its CKA is computed on CPU
with standard attention instead of FlashAttention.  On the Qwen3-4B endpoint,
the two implementations differ by at most $3\times10^{-4}$ in layerwise CKA.
The figures report full-bank estimates. Existing stability checks use 20
bootstrap resamples, each drawing 100 prompts from the pooled bank;
they quantify prompt-resampling variation rather than training-seed
uncertainty. Subset variation is appreciable in Code: the minimum CKA is
$0.9587$ on HumanEval+, $0.9940$ on MBPP+, and $0.8808$ on LiveCodeBench.
The LiveCodeBench minimum has bootstrap standard deviation $0.0459$,
so the pooled and subset estimates should be interpreted at their
respective levels of aggregation.  High CKA concerns the geometry of the pooled
prompt representations rather than equality of individual activations, since
CKA is invariant to orthogonal transformations and isotropic scaling
\citep{kornblith2019similarity}.  It also does not establish a fixed neural
tangent kernel, whose use for language-model fine-tuning requires additional
conditions \citep{malladi2023kernel}.

\paragraph{Layerwise profiles.}
Figure~\ref{fig:appendix-layerwise-cka} expands the minimum-CKA summary in
Table~\ref{tab:parameter-representation} into the complete layerwise
profiles. Panels (a)--(c) contain the same eight step-200 OPD endpoints,
and panel (d) compares the three self-RL teachers with their respective
initial students. The horizontal coordinate divides the hidden-state
index by its maximum, including the embedding and final normalized
states. Qwen3-4B provides 37 hidden-state entries and the two smaller
student architectures provide 29. We use the same pooled prompt banks
and checkpoint comparisons described above, without smoothing or
resampling the measured hidden-state indices.

The plots expose changes that a single minimum cannot localize: for
example, the minimum CKA of the Qwen3-14B Code run is 0.9850, whereas the
\QwenRLCode{} OPD endpoint stays above 0.9978. The self-RL teacher
JustRL-1.5B reaches a minimum of 0.9674 relative to its initial student,
while its OPD endpoint has minimum 0.9904. These are endpoint comparisons,
not a time series of representation learning, and high CKA does not show
that the teacher's behavior has been reproduced.

\begin{figure}[tbp]
  \centering
  \includegraphics[width=\textwidth]{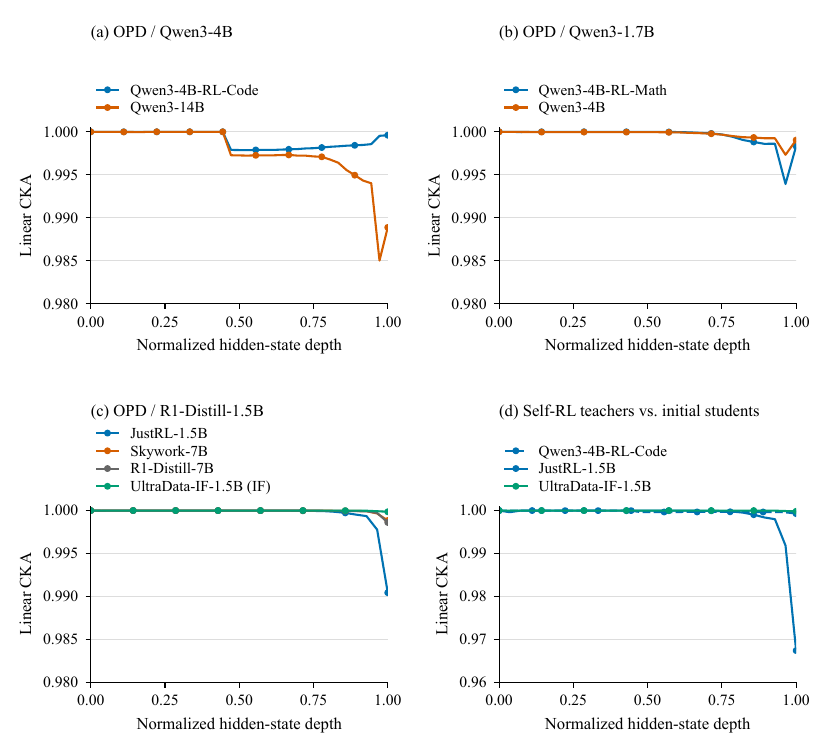}
  \caption{Layerwise linear CKA for the checkpoints in
  Table~\ref{tab:parameter-representation}. Panels (a)--(c) compare each
  OPD student after 200 updates with its initialization; the legend names
  the teacher. Panel (d) compares self-RL teachers with their initial
  students, using a wider vertical range. All R1-Distill-1.5B comparisons,
  including the IF endpoint and UltraData-IF teacher, use the same pooled
  Math prompt bank. Lines connect measured hidden-state entries; markers
  are placed every four entries for readability.}
  \label{fig:appendix-layerwise-cka}
\end{figure}

\clearpage
\section{OPD on Instruction Following}
\label{app:if-experiment}
\setcounter{table}{0}

\paragraph{Models and data.}
We compare two teachers for R1-Distill-1.5B: the self-RL teacher
UltraData-IF-1.5B, obtained by further RL training of the student on
instruction-following data \citep{fu2026rethinking}, and the larger
R1-Distill-7B. The former is the instruction-following pair included in
Table~\ref{tab:parameter-representation}; the latter is an additional
comparison reported here. Both OPD runs use the first-turn prompts of all
4,501 conversations in Multi-IF \citep{he2024multi}. All prompts remain
after the 4,096-token length filter. Evaluation uses the 4,445 conversations
with all three turns: 896 English, 518 French, 529 Hindi, 515 Russian,
454 Chinese, 493 Italian, 524 Portuguese, and 516 Spanish conversations.
All evaluation conversations' first-turn prompts occur in the training
set. These results measure transfer on overlapping inputs, not held-out
generalization; later-turn prompts are not used for OPD training.

\paragraph{Training.}
Both runs start from the same base checkpoint and use the settings in
Table~\ref{tab:if-config}. They use the student's top-16 tokens with
renormalized student weights and the original teacher--student log-ratio,
as in Eq~\eqref{eq:topk-kl}, averaged over valid response tokens. The
instruction-checking reward is logged as a diagnostic and is not mixed
into the distillation update. Each training update uses 64 prompts and eight rollouts per prompt,
so the optimizer processes 512 responses in one update. The separate
three-turn evaluation uses one response per turn.

\begin{table}[htbp]
  \centering
  \caption{Shared configuration of the two instruction-following OPD runs.
  Both runs use full-parameter updates on two H200 GPUs.}
  \label{tab:if-config}
  \begin{tabular}{@{}ll@{}}
    \toprule
    \textbf{Setting} & \textbf{Value} \\
    \midrule
    Optimizer & AdamW, $\beta=(0.9,0.999)$ \\
    Learning rate / schedule & $10^{-6}$ / constant, no warmup \\
    Weight decay / gradient clip & $0.01$ / $\ell_2$ norm 1 \\
    Prompt batch / configured minibatch (prompts) & 64 / 64 \\
    Rollouts per prompt / optimization epochs & 8 / 1 \\
    Training temperature / top-$p$ & 1 / 1 \\
    Teacher temperature & 1 \\
    Maximum prompt / response tokens & 4,096 / 7,168 \\
    Prompt order & Fixed, no shuffling \\
    Updates / checkpoint interval & 200 / 10 \\
    \midrule
    Evaluation responses per turn & 1 (seed 0) \\
    Evaluation temperature / top-$p$ & 0.6 / 0.95 \\
    Maximum evaluation response / context & 16,384 / 32,768 tokens \\
    \bottomrule
  \end{tabular}
\end{table}

\paragraph{Evaluation and score aggregation.}
We evaluate the base student, both standalone teachers, and every tenth
OPD checkpoint on the same three-turn conversations. Each subsequent turn
includes the model's preceding answers with completed reasoning blocks
removed; scoring also uses the text after the final \texttt{</think>} tag
when present. Responses without this closing tag are scored as generated.
Overlong dialogue contexts receive an empty response rather than a truncated
history. Responses that reach the output limit are retained for scoring.

For each language and turn, the Multi-IF instruction checkers provide four
rates: strict and loose prompt-level success, and strict and loose
instruction-level success. Prompt-level success requires all instructions
to pass; instruction-level success divides the total passed instructions
by the total number of instructions. Loose checks additionally allow
variants removing boundary lines or asterisks. The reported
\emph{three-turn composite score} is the unweighted mean of these four
rates over all eight languages and all three turns, expressed as a
percentage. Thus, it averages 96 language--turn--metric rates rather than
pooling conversations across languages. This is a composite benchmark
score, not a single accuracy rate or an avg@8 estimate.
Gain is the change from the common base-student score at step~0; gap
recovery divides this gain by the corresponding teacher--base score gap.
The configured single-turn diagnostic validation loop is disabled; all
reported evaluation scores come from this separate three-turn protocol.

\begin{table}[htbp]
  \centering
  \caption{Instruction-following outcomes after 200 OPD updates.
  Top: raw signed top-16 objective at update~1, its minimum over updates
  1--200, and its final value. Bottom: three-turn composite scores (\%),
  gain from the shared base student (pp), and score-gap recovery (\%).
  $^{\dagger}$ marks the self-RL teacher.}
  \label{tab:if-results}
  \begin{tabular}{@{}lrrr@{}}
    \toprule
    \textbf{Teacher} & \textbf{Initial objective} & \textbf{Minimum} & \textbf{Final} \\
    \midrule
    UltraData-IF-1.5B$^{\dagger}$ & 0.00852 & $-0.00134$ & 0.00124 \\
    R1-Distill-7B & 0.78025 & 0.09804 & 0.11119 \\
    \bottomrule
  \end{tabular}
  \par\medskip
  \begin{tabular}{@{}lrrrrr@{}}
    \toprule
    \textbf{Teacher} & \textbf{Base} & \textbf{OPD} & \textbf{Teacher score} & \textbf{Gain} & \textbf{Recovery} \\
    \midrule
    UltraData-IF-1.5B$^{\dagger}$ & 30.97\% & 40.86\% & 41.38\% & $+9.88$ & 95.0\% \\
    R1-Distill-7B & 30.97\% & 27.43\% & 47.72\% & $-3.54$ & $-21.2\%$ \\
    \bottomrule
  \end{tabular}
\end{table}

\begin{figure}[htbp]
  \centering
  \includegraphics[width=\textwidth]{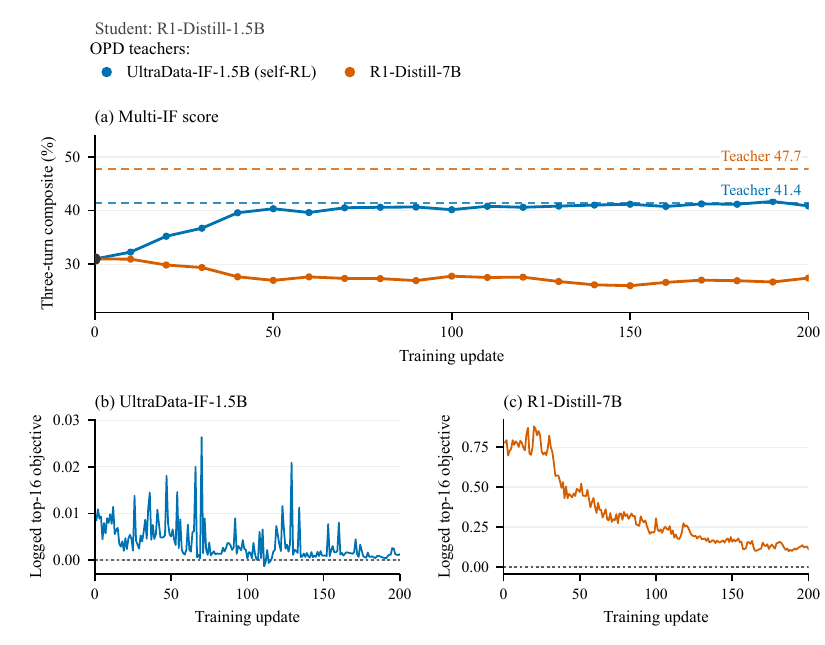}
  \caption{Instruction-following dynamics over 200 updates.
  (a) Three-turn composite scores; the black point is the shared initial
  student, and dashed lines are the corresponding standalone teachers.
  (b,c) Raw signed top-16 objectives, without smoothing, on separate linear
  axes that include zero. Colors identify the teacher used for OPD.}
  \label{fig:if-dynamics}
\end{figure}

\paragraph{Results and interpretation.}
With the self-RL teacher, the three-turn composite score rises from 30.97
at initialization to 40.86 at update~200, compared with the teacher's
41.38, closing 95.0\% of the score gap. With R1-Distill-7B, the score falls
to 27.43 despite the teacher's stronger standalone score of 47.72 and a
substantial decrease in the logged objective
(Table~\ref{tab:if-results}, Figure~\ref{fig:if-dynamics}). The third-turn
composite score follows the same pattern: 22.31 initially, 29.02 with
UltraData-IF-1.5B, and 17.62 with R1-Distill-7B; the teachers score 29.45
and 35.84, respectively.

The top-16 surrogate in Eq~\eqref{eq:topk-kl} need not be nonnegative:
its weights are renormalized on the retained support, but its log-ratio
uses the original distributions. The UltraData run has four negative
observations, with minimum $-0.00134$ at update~111. We therefore report
its raw objective values rather than interpreting a percentage reduction
above 100\% as recovery of the full-vocabulary KL. The two runs illustrate
that reducing this surrogate does not by itself ensure improvement in
instruction following. Each curve is a single training run with one
sampled dialogue per evaluation conversation, and the overlapping
training and evaluation inputs limit conclusions about generalization.

\section{Ablations}
\label{app:ablations}
\setcounter{table}{0}

\subsection{SFT on Teacher Responses and Subsequent OPD}
\label{app:ablation-sft}

We evaluate SFT for four Math student--teacher pairs from
Table~\ref{tab:student-teacher-pairs}: R1-Distill-1.5B learns from
JustRL-1.5B or R1-Distill-7B, and \QwenOneSevenNT{} learns from
\QwenRLMath{} or \QwenFourNT{}. We then continue OPD for 200 updates from
the final SFT checkpoints associated with R1-Distill-7B and
\QwenFourNT{}, using the same respective teachers. Each setting is one
training run. The \QwenRLMath{} teacher is its RL update-500 checkpoint.
SFT, direct OPD, and SFT$\rightarrow$OPD have different
training budgets; the last includes the entire SFT stage before its
200 OPD updates.

\paragraph{Teacher-response collection.}
We use the project's DAPO-Math-17K training file, containing 17,917 rows
and 17,398 distinct question texts. We preserve each question body,
replace the source file's answer-format instructions with
\emph{Please reason step by step, and put your final answer within
\texttt{\textbackslash boxed\{\}}.}, and apply the teacher's chat template.
Each teacher contributes 20,000 accepted responses. Since there are fewer
than 20,000 distinct questions, collection traverses deterministically
shuffled cycles of the source rows, with shuffle seed $42+\text{cycle}$
and fresh generations on repeated questions. The candidate pool
contains four cycles (71,668 slots). We draw one response per
attempt, retry a candidate slot at most three times, and continue through
the candidate pool until the accepted-response quota is reached.

Teacher generation uses temperature 0.7, top-$p$ 0.95, no top-$k$
truncation, repetition penalty 1, a 12,288-token response cap, and a
32,768-token context limit. Two vLLM workers use seeds 42 and 43, with generation batches of 128.
The Qwen teachers generate in non-thinking mode; the R1 teachers retain
their reasoning traces. We require the output to contain the
\texttt{\textbackslash boxed} marker and terminate normally, reject
detected repetitive text, and additionally require a closing
\texttt{</think>} tag for R1 responses. The repetition filters detect
a stripped line of at least 20 characters appearing five times, a
100-character substring appearing three times when sampled every ten
characters, or three consecutive repetitions of a 50-character block
for outputs longer than 5,000 characters. We do not filter for answer
correctness. A final student-tokenization check rejects examples longer
than 14,336 tokens including prompt and response; all accepted examples
fit this limit, so the admitted targets need not be truncated.

\begin{table}[htbp]
  \centering
  \normalsize
  \setlength{\tabcolsep}{4pt}
  \caption{SFT response collection and row-split overlap. Every teacher
  supplies 20,000 accepted responses, split into 19,000 training and
  1,000 held-out rows. \emph{Distinct questions} counts the accepted pool.
  \emph{Shared held-out rows} counts held-out examples whose question text also
  appears in the training split.}
  \label{tab:sft-response-data}
  \begin{tabular}{@{}lrrr@{}}
    \toprule
    Teacher & Attempts & Distinct questions & Shared held-out rows \\
    \midrule
    JustRL-1.5B & 27,706 & 16,401 & 344 \\
    R1-Distill-7B & 30,565 & 15,855 & 387 \\
    \QwenRLMath{} & 32,034 & 15,481 & 424 \\
    \QwenFourNT{} & 22,931 & 17,123 & 268 \\
    \bottomrule
  \end{tabular}
\end{table}

\paragraph{SFT objective and implementation.}
We perform full-parameter SFT with LLaMA-Factory, using cross-entropy on
the teacher response and masking prompt and padding tokens. We use
source revision \texttt{ac26e38} of LLaMA-Factory. Sequence packing
and label smoothing are disabled. The recorded templates are \texttt{deepseekr1}
for R1 and \texttt{qwen3} for Qwen, with
\texttt{enable\_thinking=True} in the SFT preprocessing configuration.
For R1, a missing opening \texttt{<think>} tag is restored before
training, preserving the generated reasoning and answer. For Qwen,
the reasoning template prepends empty thinking tags to a response that
has neither thinking tag; those tags are included in the supervised
target. This preprocessing flag is distinct from the non-thinking
teacher-generation and Qwen benchmark settings.

We use BF16 training, DeepSpeed ZeRO-2, FlashAttention-2, the Liger kernel,
and gradient checkpointing on two H200 GPUs.
Table~\ref{tab:sft-training-settings} gives the optimization settings.
The response rows are split with seed 42 before preprocessing and
shuffled during SFT. This is
a row-level split: repeated questions can occur on both sides
(Table~\ref{tab:sft-response-data}). Consequently, held-out
cross-entropy is a diagnostic of fitting these teacher responses,
not an independent test of mathematical generalization. We measure
mathematical capability separately on the benchmarks below.

\begin{table}[htbp]
  \centering
  \normalsize
  \setlength{\tabcolsep}{4pt}
  \caption{Optimization settings for all four SFT runs and the two
  subsequent OPD runs. OPD response limits differ by student; the other
  listed settings are shared within each stage.}
  \label{tab:sft-training-settings}
  \begin{tabular}{@{}p{126pt}p{119pt}p{129pt}@{}}
    \toprule
    \raggedright Setting & \raggedright SFT & \raggedright Subsequent OPD \tabularnewline
    \midrule
    \raggedright Training examples & \raggedright 19,000 response rows & \raggedright DAPO training prompts \tabularnewline
    \raggedright Optimizer updates & \raggedright 297 (one epoch) & \raggedright 200 \tabularnewline
    \raggedright Global batch & \raggedright 64 examples & \raggedright 64 prompts, one rollout each \tabularnewline
    \raggedright Per-GPU SFT batch & \raggedright 8, accumulated 4 times & \raggedright Dynamic token batching \tabularnewline
    \raggedright Trainable weights & \raggedright All & \raggedright All \tabularnewline
    \raggedright Optimizer & \raggedright Fused AdamW & \raggedright AdamW \tabularnewline
    \raggedright Learning rate & \raggedright Peak $10^{-5}$ & \raggedright Constant $10^{-6}$ \tabularnewline
    \raggedright Schedule / warmup & \raggedright Cosine / 5\% & \raggedright Constant / none \tabularnewline
    \raggedright Adam $(\beta_1,\beta_2)$ & \raggedright $(0.9,0.999)$ & \raggedright $(0.9,0.999)$ \tabularnewline
    \raggedright Adam $\epsilon$ & \raggedright $10^{-8}$ & \raggedright Library default (not explicitly set) \tabularnewline
    \raggedright Weight decay & \raggedright 0 & \raggedright 0.01 \tabularnewline
    \raggedright Gradient clipping norm & \raggedright 1 & \raggedright 1 \tabularnewline
    \raggedright Prompt + response limit & \raggedright 14,336 tokens total & \raggedright 1,024 prompt tokens; 21,000 response tokens (R1), 7,168 (Qwen) \tabularnewline
    \raggedright Random / data seed & \raggedright 42 / 42 & \raggedright Not explicitly set in saved configuration \tabularnewline
    \raggedright Loss recording & \raggedright Every 5 updates & \raggedright Every update \tabularnewline
    \raggedright Evaluation during training & \raggedright Row-heldout CE at 100, 200, 297 & \raggedright Math benchmarks at $0,10,\ldots,200$ \tabularnewline
    \bottomrule
  \end{tabular}
\end{table}

\paragraph{SFT loss dynamics.}
Figure~\ref{fig:sft-training-dynamics} shows all four runs. The first
logged training losses at update 5 are 0.4108, 0.3711, 0.4262, and
0.2764 for JustRL-1.5B, R1-Distill-7B, \QwenRLMath{}, and
\QwenFourNT{}, respectively. The last logged training losses, at update
295, are 0.2951, 0.3459, 0.3229, and 0.2368. Training finishes at update
297; the corresponding final held-out losses are 0.2942, 0.3537,
0.3252, and 0.2332. These cross-entropies concern different
teacher-response datasets and are not the OPD objective. The final
checkpoint is used for benchmarking and, where applicable, subsequent
OPD; checkpoints are not selected by benchmark accuracy.

\begin{figure}[!t]
  \centering
  \includegraphics[width=\textwidth]{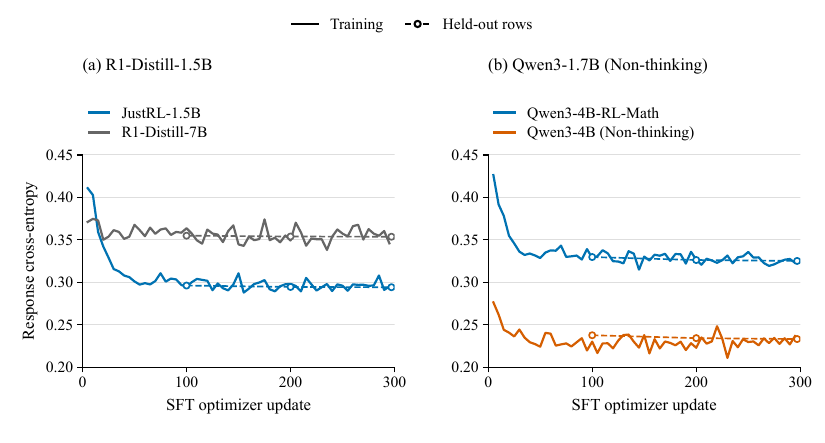}
  \caption{SFT loss dynamics for all four teacher-response datasets.
  Solid lines join recorded training cross-entropies at updates
  $5,10,\ldots,295$; no update-0 loss is available.
  Dashed lines and hollow circles mark held-out response-row
  cross-entropy at updates 100, 200, and 297. The held-out split has
  question overlap with training and is not a mathematical benchmark.
  Colors identify teachers as in Figure~\ref{fig:teacher-trainability}.
  Each teacher contributes one run; no averaging across seeds is shown.}
  \label{fig:sft-training-dynamics}
\end{figure}

\paragraph{Mathematical benchmark evaluation.}
We evaluate each final SFT checkpoint on 30 AIME 2024, 30 AIME 2025,
and 83 AMC 2023 problems. We sample eight responses per problem
(1,144 responses total), with temperature 0.7, top-$p$ 0.95, and a
31,744-token response cap, 1,024-token prompt limit, and 32,768-token
context limit. We score answer correctness and report the
unweighted mean of the three benchmark accuracies, using the same
\(\mathrm{avg}@8\) convention as Section~\ref{sec:phenomenon-setup}.
R1 evaluation uses vLLM 0.11.0 and the archived R1/Modal benchmark
prompts; Qwen uses SGLang, its non-thinking chat template, and the
Qwen benchmark prompts. These benchmarks are separate from the
1,000-row SFT loss holdout. No benchmark-accuracy trajectory was
recorded during the SFT stage; Figure~\ref{fig:sft-training-dynamics}
therefore reports only SFT cross-entropy.

\begin{table}[htbp]
  \centering
  \normalsize
  \setlength{\tabcolsep}{4pt}
  \caption{Standalone evaluation of the four final SFT checkpoints.
  Every entry is a percentage; Mean is the unweighted mean of the
  three benchmark columns. These are the original SFT evaluations,
  not the subsequent OPD runs' independent step-0 evaluations.}
  \label{tab:sft-benchmark-breakdown}
  \begin{tabular}{@{}lrrrr@{}}
    \toprule
    Teacher & AIME24 & AIME25 & AMC23 & Mean \\
    \midrule
    JustRL-1.5B & 46.67\% & 35.83\% & 79.67\% & 54.06\% \\
    R1-Distill-7B & 21.67\% & 20.42\% & 55.57\% & 32.55\% \\
    \QwenRLMath{} & 31.25\% & 27.92\% & 64.01\% & 41.06\% \\
    \QwenFourNT{} & 13.33\% & 12.92\% & 41.42\% & 22.56\% \\
    \bottomrule
  \end{tabular}
\end{table}

\paragraph{Subsequent OPD.}
The two continuation runs use the final update-297 SFT models as their
students and retain the corresponding teachers. We optimize the
student-top-16 reverse-KL surrogate in Eq~\eqref{eq:topk-kl}, with
student-probability weighting and a global mean over valid response
tokens. Both student and teacher temperatures are 1; training uses
top-$p=1$, no top-$k$ sampling truncation, and repetition penalty 1.
Each update uses 64 prompts, one student rollout per prompt, and one
optimization pass. The training data are not shuffled. Optimizer and
length settings are listed in Table~\ref{tab:sft-training-settings}.
Benchmark decoding and aggregation follow the preceding paragraph,
with an evaluation before training and every ten updates thereafter.

Both continuation runs use two H200 GPUs. The Qwen continuation uses
SGLang, as does its direct-OPD comparison. The R1 continuation uses
vLLM 0.11.0 in a different runtime from the four-GPU direct-OPD run;
the global batch remains 64. Its validation sampling uses vLLM's
PyTorch sampler rather than the original FlashInfer sampler, so the
random draws need not coincide. Checkpoints are saved every ten
updates, with at most two retained, and the comparison uses the
update-200 endpoint. The Qwen log includes an abandoned initial
attempt at steps 0 and 1; the completed trajectory supersedes those
records. Figures and loss summaries use the last recorded value at
each step, yielding one evaluation at each scheduled checkpoint and
200 training-loss observations per continuation.

\paragraph{Separate post-SFT starting evaluations.}
The continuation runs independently evaluate their SFT starting
checkpoints, obtaining 36.57\% for R1 and 21.63\% for Qwen, compared
with the original standalone SFT scores of 32.55\% and 22.56\%.
For Qwen, all 143 rendered inputs and ground truths match the original
SFT benchmark in order; the responses are newly sampled. For R1,
the archived standalone benchmark asks
\emph{Let's think step by step and output the final answer within
\texttt{\textbackslash boxed\{\}}.}, whereas the continuation
uses \emph{Please reason step by step, and put your final answer within
\texttt{\textbackslash boxed\{\}}.}
The R1 scores therefore come from different prompt formats as well
as different sampled responses and runtimes. They must not be
treated as repeated evaluations under identical conditions or as
the base student's pre-SFT accuracy.

\paragraph{Accuracy comparison.}
Table~\ref{tab:sft-ablation} compares the final outcomes. With the three
larger-scale teachers tested, SFT alone does no better than direct OPD.
For the self-RL teacher JustRL-1.5B, SFT reaches 54.1\%, slightly above
direct OPD. Continuing OPD after SFT reaches 40.2\% with R1-Distill-7B
and 22.6\% with \QwenFourNT{}, compared with 42.7\% and 22.7\% for
direct OPD. Relative to their own post-SFT starting evaluations, the
continuations gain 3.60 and 1.01 percentage points. Neither surpasses
direct OPD at the final update. These single-run comparisons do not
establish statistical significance for the accuracy differences.

\begin{table}[htbp]
  \centering
  \normalsize
  \setlength{\tabcolsep}{3pt}
  \caption{SFT on teacher responses, with and without subsequent OPD.
  Start is the base student's step-0 accuracy before either training stage;
  OPD is its accuracy after 200 direct OPD updates. SFT is the original
  evaluation of the final SFT checkpoint; SFT$\rightarrow$OPD is the
  accuracy after 200 additional OPD updates. All accuracies are percentages
  averaged across the three Math benchmarks. A dash indicates an untested
  setting. $^{\dagger}$ marks the self-RL teacher. Qwen models use
  non-thinking mode for generation and evaluation.}
  \label{tab:sft-ablation}
  \begin{tabular}{@{}llrrrrr@{}}
    \toprule
    Student & Teacher & Start & OPD & SFT & SFT$\rightarrow$OPD & Teacher \\
    \midrule
    R1-Distill-1.5B & JustRL-1.5B$^{\dagger}$ & 39.9 & 51.9 & 54.1 & --- & 56.9 \\
     & R1-Distill-7B & 39.9 & 42.7 & 32.6 & 40.2 & 57.8 \\
    \midrule
    Qwen3-1.7B & \QwenRLMath{} & 21.1 & 43.9 & 41.1 & --- & 69.2 \\
     & Qwen3-4B & 20.7 & 22.7 & 22.6 & 22.6 & 35.1 \\
    \bottomrule
  \end{tabular}
\end{table}

\paragraph{OPD loss dynamics after SFT.}
Figure~\ref{fig:sft-opd-training-dynamics} and
Table~\ref{tab:sft-opd-loss} show that neither continuation sustains
loss reduction toward zero. The maximum observed reductions are
37.1\% and 23.0\%; the Qwen final loss is above its initial OPD loss.
Thus, SFT initialization does not remove the limited loss reduction
or the remaining performance gap in these two settings. This finding
concerns the tested data, recipe, and budgets, rather than ruling out
other SFT initializations.

\begin{table}[htbp]
  \centering
  \normalsize
  \setlength{\tabcolsep}{4pt}
  \caption{Loss during the 200 OPD updates following SFT. The initial loss
  is measured at the first OPD update after SFT. Loss reduction is
  $(1-\widehat{\mathcal L}_{\min}/\widehat{\mathcal L}_{\mathrm{start}})
  \times100\%$, computed from unrounded, unsmoothed losses.}
  \label{tab:sft-opd-loss}
  \begin{tabular}{@{}lrrrc@{}}
    \toprule
    Teacher & $\widehat{\mathcal L}_{\mathrm{start}}$
    & $\widehat{\mathcal L}_{\min}$ & $\widehat{\mathcal L}_{\mathrm{end}}$
    & Loss reduction \\
    \midrule
    R1-Distill-7B & 0.1461 & 0.0919 & 0.1232 & 37.1\% \\
    \QwenFourNT{} & 0.0805 & 0.0620 & 0.0861 & 23.0\% \\
    \bottomrule
  \end{tabular}
\end{table}

\begin{figure}[!t]
  \centering
  \includegraphics[width=\textwidth]{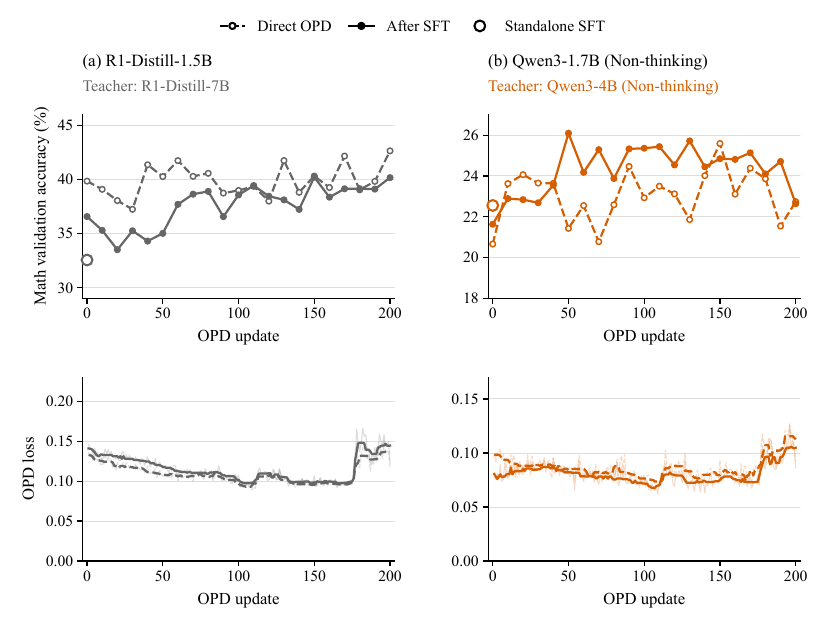}
  \caption{Direct OPD and OPD after SFT with the two tested teachers.
  Top: benchmark accuracy from each run's own recorded evaluations.
  The direct R1 curve's step-0 point is the shared base-model
  measurement described in Appendix~\ref{app:teacher-comparison-details}.
  Larger hollow circles separately show the original standalone SFT
  evaluations; they are not connected to either trajectory.
  The R1 standalone SFT evaluation uses a different prompt suffix.
  Bottom: raw per-update OPD loss in faint lines and centered
  nine-update rolling medians in stronger lines. Solid lines are
  continuations after SFT; dashed lines are direct OPD from the base
  student. Each panel uses a linear scale containing its observations.
  The two training stages have separate update counts, and the
  continuation does not include the SFT cost on its horizontal axis.}
  \label{fig:sft-opd-training-dynamics}
\end{figure}

\subsection{Vocabulary support and training rollouts}
\label{app:support-rollout-ablations}

We compare the main top-16, one-rollout configuration with full-logit OPD
at one rollout per prompt and top-16 OPD at eight rollouts per prompt.
These comparisons cover Qwen3-1.7B with \QwenRLMath{} and Qwen3-4B,
and R1-Distill-1.5B with JustRL-1.5B and R1-Distill-7B.
Table~\ref{tab:support-rollout-results} reports all twelve conditions:
four main-run baselines and eight additional completed runs.
The two Qwen full-logit runs and the Qwen3-4B eight-rollout run are the
same runs already reported in Table~\ref{tab:qwen-math-200-ablation},
not additional replications. JustRL-1.5B is a self-RL teacher for
R1-Distill-1.5B; \QwenRLMath{} is an RL-trained larger model relative
to Qwen3-1.7B.

\paragraph{Data and optimization.}
All additional runs use the same DAPO-Math-17K training file as their
corresponding baselines, containing 17,917 prompts before prompt-length
filtering. We apply the student's chat template, filter prompts longer
than 1,024 tokens, and keep the dataset order fixed. Qwen3 uses
non-thinking mode. Training response limits are 7,168 tokens for
Qwen3-1.7B and 21,000 for R1-Distill-1.5B. Each run has 200 optimizer
updates, with 64 prompts per update, student sampling temperature $1$,
top-$p=1$, no sampling top-$k$ cutoff, and teacher temperature $1$.
The optimizer is AdamW with constant learning rate $10^{-6}$, no warmup,
$(\beta_1,\beta_2)=(0.9,0.999)$, weight decay $0.01$, and gradient-norm
clipping at $1$. We use one optimization epoch per prompt batch and a
global mean over valid response tokens. The teacher remains fixed;
no task reward or format reward is added to the OPD update.
The vocabulary top-16 restriction is separate from the unrestricted
training sampling distribution.

Increasing the rollout count from one to eight changes the number of
sampled responses per update from 64 to 512. The optimization minibatch
is expanded accordingly, so both settings still perform one optimizer
update per batch. Thus, the total numbers of sampled training responses
are 12,800 and 102,400, respectively, while both process 12,800 prompt
instances. This is a comparison at fixed prompt and update budgets,
not at fixed generation cost. We use dynamic microbatching with a
32,768-token per-GPU budget, FSDP, activation checkpointing, and bfloat16
rollout inference. The additional runs use two H200 GPUs, except the
R1-Distill-7B top-16, eight-rollout run, which uses four. The Qwen
\QwenRLMath{} top-16 baseline and both R1 top-16 baselines use four GPUs.
All other Qwen conditions use two. Qwen runs use SGLang 0.5.3rc0 and
Transformers 4.56.1; the additional R1 runs use vLLM 0.11.0 and
Transformers 4.57.6. The additional runs use PyTorch 2.8.0 and Ray 2.49.2.
No separate experiment-level seed is recorded in the saved data
configuration. Each condition is a single stochastic training run;
the comparisons are not paired-seed or repeated-seed estimates.

\paragraph{Full-logit implementation.}
Top-16 OPD selects the student's sixteen highest-probability vocabulary
entries at each visited prefix. Student probabilities are renormalized
over these entries to form detached weights, while the log-probability
ratio uses each model's full-vocabulary normalization. Full-logit OPD
instead includes all 151,936 student output entries. It retains the
same detached student-weighted update and loss normalization, computing
the output projection in chunks of 128 prefix states to limit memory
use. No vocabulary truncation is applied in these full-logit runs.
For R1-Distill-7B, whose output head has 152,064 entries, teacher logits
are normalized over all 152,064 entries before selecting the 151,936
entries aligned with the student; the selected teacher probabilities
are not renormalized. Token-ID compatibility is checked before training.
Because the top-16 and full-logit objectives differ, we compare their
loss reductions relative to their own update-1 losses, rather than
treating their absolute losses as the same quantity.

\paragraph{Evaluation and run accounting.}
All additional runs are evaluated before training and every ten updates
on AIME 2024 (30 problems), AIME 2025 (30), and AMC 2023 (83).
Each evaluation contains 1,144 responses: eight per problem, sampled
at temperature $0.7$, top-$p=0.95$, and a 31,744-token response limit
within a 32,768-token context. We average correctness over responses
within each benchmark and then average the three benchmark scores
equally. The same answer-extraction and mathematical-correctness scorer
is used throughout; format reward is disabled. We retain each student
group's original prompt formatting: Qwen validation requests step-by-step
reasoning and a boxed final answer under its non-thinking chat template,
while R1 validation uses the corresponding R1 prompt and reasoning
prefix. Dataset and rendered-prompt checks verify consistency within
each student group. The two original R1 baselines use the shared
initial-student evaluation described in
Appendix~\ref{app:teacher-comparison-details}; every additional run uses
its own step-0 evaluation. Variations in these initial estimates reflect
sampling from the same initial checkpoint.

For each additional run, the loss record contains all 200 updates and
the validation record contains all 21 evaluation checkpoints. Native
optimizer and data-loader counters were checked against the saved
checkpoints. We use the completed run's canonical local records when
an earlier attempt failed: in particular, the JustRL eight-rollout run
was restarted from the initial checkpoint after an unsaved one-update
attempt, and the discarded attempt is not concatenated with the
completed trajectory. These eight additional runs are excluded from
the seven-run averages in the main text.

\begin{table}[tbp]
\centering
\normalsize
\setlength{\tabcolsep}{3pt}
\renewcommand{\arraystretch}{1.18}
\caption{Support and rollout comparisons after 200 updates. Accuracy is
macro-averaged $\mathrm{avg}@8$. Maximum and final loss reduction use
$(1-\widehat{\mathcal L}_{\min}/\widehat{\mathcal L}_1)\times100\%$ and
$(1-\widehat{\mathcal L}_{200}/\widehat{\mathcal L}_1)\times100\%$,
respectively. Each uses its own run's objective. The four top-16, $n=1$
rows are the main-run baselines; the other eight rows are additional runs.}
\label{tab:support-rollout-results}
\begin{tabular}{@{}llcrrrr@{}}
\toprule
Teacher & Support & $n$ & \multicolumn{2}{c}{Accuracy} & \multicolumn{2}{c}{Loss reduction} \\
\cmidrule(lr){4-5}\cmidrule(l){6-7}
 & & & Initial & Final & Maximum & Final \\
\midrule
\multicolumn{7}{@{}l}{\textit{Student: Qwen3-1.7B (Non-thinking)}} \\
\addlinespace[2pt]
\QwenRLMath{} & Top-16 & 1 & 21.10\% & 43.94\% & 65.08\% & 55.37\% \\
 & Full & 1 & 20.55\% & 46.36\% & 62.68\% & 54.85\% \\
 & Top-16 & 8 & 22.95\% & 45.21\% & 60.62\% & 51.29\% \\
\addlinespace[3pt]
Qwen3-4B & Top-16 & 1 & 20.66\% & 22.74\% & 30.36\% & 6.24\% \\
 & Full & 1 & 21.66\% & 22.88\% & 37.52\% & 11.60\% \\
 & Top-16 & 8 & 21.85\% & 24.87\% & 33.02\% & 9.89\% \\
\midrule
\multicolumn{7}{@{}l}{\textit{Student: R1-Distill-1.5B}} \\
\addlinespace[2pt]
JustRL-1.5B & Top-16 & 1 & 39.85\% & 51.88\% & 96.29\% & 92.95\% \\
 & Full & 1 & 39.85\% & 50.46\% & 96.22\% & 91.80\% \\
 & Top-16 & 8 & 39.85\% & 51.93\% & 96.18\% & 31.09\% \\
\addlinespace[3pt]
R1-Distill-7B & Top-16 & 1 & 39.85\% & 42.66\% & 34.35\% & 13.94\% \\
 & Full & 1 & 39.85\% & 40.32\% & 30.77\% & -7.05\% \\
 & Top-16 & 8 & 40.39\% & 38.87\% & 35.03\% & 18.29\% \\
\bottomrule
\end{tabular}
\end{table}

\begin{figure}[!t]
\centering
\includegraphics[width=\textwidth]{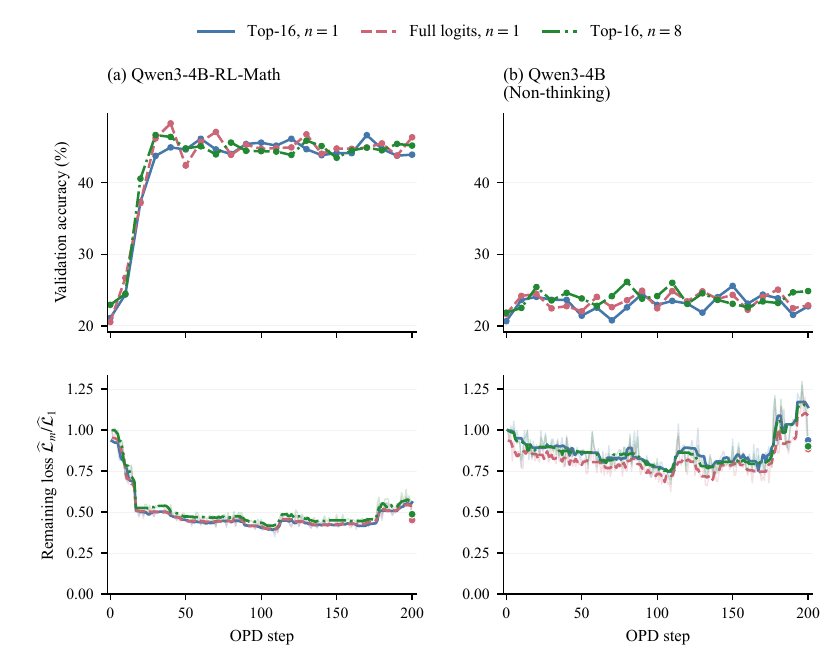}
\caption{Vocabulary-support and rollout comparisons for Qwen3-1.7B
(Non-thinking). Each column fixes the OPD teacher. Top: validation
accuracy. Bottom: loss divided by that run's update-1 loss. Faint lines
show all raw losses; thicker lines show centered nine-step rolling
medians. Filled loss markers show the raw update-200 values.
Here $n$ is the number of training rollouts per prompt; all validation
scores use $\mathrm{avg}@8$.}
\label{fig:support-rollout-qwen}
\end{figure}

\paragraph{Observed behavior.}
Figures~\ref{fig:support-rollout-qwen} and~\ref{fig:support-rollout-r1}
show validation accuracy and loss over the full training window.
For Qwen3-1.7B, final accuracy with \QwenRLMath{} is $43.94\%$,
$46.36\%$, and $45.21\%$ for top-16 with one rollout, full logits with
one rollout, and top-16 with eight rollouts, respectively. The
corresponding final loss reductions are $55.37\%$, $54.85\%$, and
$51.29\%$. With Qwen3-4B, accuracy remains much lower
($22.74\%$, $22.88\%$, and $24.87\%$), and final loss reductions remain
small ($6.24\%$, $11.60\%$, and $9.89\%$). Thus, removing the vocabulary
restriction or increasing the rollout count preserves the broad
teacher-dependent behavior in these Qwen runs.

\begin{figure}[!t]
\centering
\includegraphics[width=\textwidth]{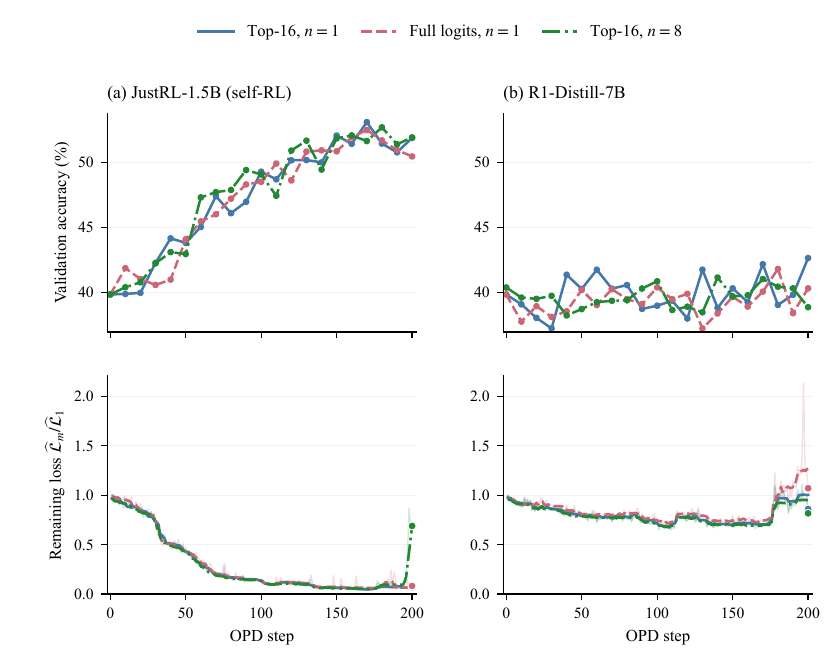}
\caption{Vocabulary-support and rollout comparisons for R1-Distill-1.5B,
with the same plotting conventions as
Figure~\ref{fig:support-rollout-qwen}. JustRL-1.5B is a self-RL teacher.
The late loss increase in its eight-rollout run is retained, as is the
loss increase in the full-logit R1-Distill-7B run.}
\label{fig:support-rollout-r1}
\end{figure}

For R1-Distill-1.5B, JustRL-1.5B also produces higher final accuracy
than R1-Distill-7B under all three configurations. With JustRL,
final accuracies are $51.88\%$, $50.46\%$, and $51.93\%$, and maximum
loss reductions are all approximately $96\%$. However, the eight-rollout
run's loss rises during its final updates, leaving a final loss reduction
of $31.09\%$, compared with $92.95\%$ and $91.80\%$ in the one-rollout
top-16 and full-logit runs. The R1-Distill-7B full-logit run ends with
a loss $7.05\%$ above its update-1 value; its maximum loss reduction
is $30.77\%$. We therefore report both maximum and final loss reduction.
These runs preserve the broad teacher ordering, but their detailed loss
trajectories are not interchangeable, and a larger rollout count does
not uniformly improve optimization. In the JustRL eight-rollout run,
the pre-clip gradient norm and mean response length also rise during
the last updates; the epoch counter remains zero throughout this
interval. These logs establish the late increase but do not identify
its cause.

\end{document}